\documentclass[twoside]{article}
\usepackage[accepted]{aistats2024}
\usepackage[round]{natbib}

\usepackage{times}
\usepackage{soul}
\usepackage{url}
\usepackage[hidelinks]{hyperref}
\usepackage[utf8]{inputenc}
\usepackage[small]{caption}
\usepackage{subcaption}
\usepackage{graphicx}
\usepackage{amsmath}
\usepackage{amsthm}
\usepackage{booktabs}
\usepackage[ruled,vlined,linesnumbered]{algorithm2e}
\usepackage[switch]{lineno}%
\usepackage{physics}
\usepackage{soul}
\usepackage{amsfonts, amssymb, amstext, bbm}
\usepackage{mathtools}
\usepackage{cleveref}

\DeclareMathOperator*{\argmax}{arg\,max}
\DeclareMathOperator*{\argmin}{arg\,min}
\def\x{{\bf x}}

\def\z{{\bf z}}
\def\A{{\bf A}}
\def\e{{\bf e}}
\def\b{{\bf b}}
\def\B{{\bf B}}
\def\C{{\bf C}}
\def\D{{\bf D}}
\def\y{{\bf y}}
\def\U{{\bf U}}
\def\O{{\bf O}}
\def\v{{\bf v}}
\def\u{{\bf u}}
\def\I{{\bf I}}

\def\w{{\bf w}}

\def\P{{\bf P}}

\def\nul{{\bf 0}}

\def\thetabf{{\boldsymbol \theta}}
\def\Thetabf{{\boldsymbol \Theta}}

\def\etabf{\boldsymbol{\eta}}

\def\Deltabf{\boldsymbol{\Delta}}

\def\mubf{\boldsymbol{\mu}}
\def\Sigmabf{\boldsymbol{\Sigma}}
\def\epsilonbf{\boldsymbol{\epsilon}}
\def\xibf{\boldsymbol{\xi}}

\def\Rset{{\mathbb R}}

\newtheorem{assumption}{Assumption}
\newtheorem{remark}{Remark}
\newtheorem{theorem}{Theorem}
\newtheorem{lemma}{Lemma}
\newtheorem{definition}{Definition}

\begin{document}

%

%

\twocolumn[

\aistatstitle{Meta Learning in Bandits within Shared Affine Subspaces}

\aistatsauthor{ Steven Bilaj \And Sofien Dhouib \And  Setareh Maghsudi }

\aistatsaddress{ Ruhr University Bochum \And  University of Tübingen \And Ruhr University Bochum } ]

\begin{abstract}
  We study the problem of meta-learning several contextual stochastic bandits tasks by leveraging their concentration around a low-dimensional affine subspace, which we learn via online principal component analysis to reduce the expected regret over the encountered bandits. We propose and theoretically analyze two strategies that solve the problem: One based on the principle of optimism in the face of uncertainty and the other via Thompson sampling. Our framework is generic and includes previously proposed approaches as special cases. Besides, the empirical results show that our methods significantly reduce the regret on several bandit tasks.
\end{abstract}

\section{Introduction}

In several real-world applications, such as website design and healthcare, the system recommends an item to a user upon observing some side information depending on the user and the corresponding item. Upon receiving the recommendation, the user sends feedback to the system that captures his interest in the recommendation \citep{glowacka2019bandit,bouneffouf2020survey,Atan23:DDO}. One can interpret the feedback as a reward that characterizes the suitability of the selected recommendation or action with the final objective of maximizing the cumulative payoff over time. At the same time, such a selection might be suboptimal due to the incomplete knowledge of the environment. This \emph{exploration-exploitation} tradeoff, along with the side information, is formalized by the \emph{contextual multi-armed bandit (CMAB)} problem \citep{langford2007epoch,li2010contextual,chu2011contextual,abbasi2011improved,Nourani22:LCS}, a notable extension of the \emph{multi-armed bandit (MAB)} problem \citep{thompson1933likelihood,robbins1952some}.

\par In the applications mentioned above, the tasks often relate to each other despite being different. For instance, subgroups of patients have comparable features. As another example,  holidays or discount periods promote similar interests in the products of an e-commerce website. That observation motivates us to look beyond a single task to uncover a relation between different ones to accelerate learning on newly encountered tasks. That problem, referred to as \emph{meta-learning} or \emph{learning-to-learn (LTL)}, has mainly appeared in the offline learning literature so far \citep{hutter2019automated}. Nevertheless, an emergent body of literature combines LTL and MAB to accelerate learning and reduce the average regret per task \citep{cella2020meta,cella2021multi,Bilaj23:HTB}. In the linear contextual setting, an assumption about the preference vectors captures the relation between the tasks.

\par In this work, we assume that the feature vectors stem from a distribution that concentrates around a low dimensional subspace, \emph{i.e.,} its variance is explained by a limited number of principal components. We propose learning this structure using online \emph{principal component analysis (PCA)}. We then exploit that knowledge to develop two decision-making policies. The first policy relies on the principle of optimism in the face of uncertainty for linear bandits (OFUL) \citep{chu2011contextual,abbasi2011improved}, and the second is a Thompson sampling policy \citep{russo2018tutorial,agrawal:linearTS}. Analytically, we establish per-task regret upper bounds for both strategies that theoretically prove the benefit of learning such a structure. Moreover, our empirical evaluations of our methods using simulated and real-world data sets confirm their benefits.

Our paper is organized as follows. We review meta-learning and related themes for bandit problems in \Cref{sec:related_work}. Then we formulate our problem in \Cref{sec:pb_formulation}. We describe the subspace learning procedure in \Cref{sec:subspace_learning} to use in our proposed algorithms in \Cref{sec:proj_linucb,sec:proj_TS}. Finally, we empirically assess our algorithms in \Cref{sec:experiments}.

\section{Related Work}
\label{sec:related_work}
Learning to learn was first developed for offline learning \citep{thrun1998lifelong,baxter2000model,hutter2019automated} as a sub-field of transfer learning. In this paradigm, one seeks to learn a structure shared by many tasks to generalize to new ones. That structure can be encoded in several ways such as a prior over the task distribution \citep{amit2018meta,rothfuss2021pacoh}, a kernel \citep{aioli2012transferkernellearning}, a common mean around which tasks concentrate \citep{denevi2018learning}, or an approximate low dimensional manifold \citep{jiang2022subspace}, to name a few.

Recently, meta-learning received attention in the online setting \citep{finn2019online}, more precisely, in the case of bandit feedback. The main idea is that the learner interacts sequentially with bandit problems, so the meta-learned shared structure accelerates exploration for upcoming tasks. In this setting, the objective is to improve the regret guarantees compared to those achievable by considering each task separately. The notion of regret can capture such guarantees; nevertheless, it has several definitions depending on the line of work. We distinguish mainly two regret types in a multi-task scenario: \textit{transfer regret} and \textit{meta regret}. The former depends on the number of learned tasks, whereas the latter takes an expectation on a possibly infinite number of tasks.

\par Concerning transfer regret, the goal is to prove sub-linear regret in the number of tasks. If the learner considers each task independently, the total regret over tasks is linear in the number of tasks. Within a task, the expected transfer regret is linear in the number of rounds.
References \citet{cella2021multi} and \citet{cella2022meta} prove that if preference vectors have a low-rank structure, then learning it improves performance. 

In the setting of Bayesian bandits, instead of assuming that the agent knows the true prior over tasks, a recent line of work proposes to learn that distribution. For example, \citet{bastani2019Meta} studies the dynamic pricing problem and proposes a Thompson sampling approach. Reference \citet{kveton2021meta} generalizes the scope of the stochastic MAB problem by developing a meta-Thompson Sampling (meta-TS) algorithm. \cite{basu2021No} improves the guarantees of  \citet{kveton2021meta} via a modification of meta-TS. It also generalizes the core idea to other bandit settings, such as linear and combinatorial bandits. While \cite{basu2021No} and \citet{kveton2021meta} study learning the mean of the tasks with a known covariance matrix, \citet{peleg2022Metalearning} relaxes that assumption. It proposes a general multivariate Gaussian prior learning framework that applies to several prior-update-based bandit algorithms. In the nonlinear contextual bandit case, \citet{kassraie2022MetaLearning,schur2022Lifelong} investigate learning a shared kernel. Concerning the second type of guarantees, \citet{cella2020meta} proves that the regret expectation over a potentially infinite number of tasks shrinks to 0 provided that the ridge regularization parameter is inversely proportional to the tasks' variance, and that said variance approaches 0.
\par Another line of work \citep{boutilier2020differentiable,kveton2020meta,yang2020differentiable} takes inspiration from the policy gradient methods \citep{williams1992simple} and aims to learn hyperparameters of policies to maximize the expected cumulative reward. Besides, meta-learning is also applicable to solve problems in other settings concerning the reward generating mechanism, such as the non-stationary case \citep{azizi2022non}, and more generally the adversarial case \citep{balcan2022meta}.
\par Multi-task learning is a field closely related to meta-learning. The main difference between the two is the following: The former is about simultaneously learning over a finite family of bandit tasks without being concerned with generalization over future ones. That method is applied to solve the unstructured stochastic bandit case \citep{azar2013sequential}, where although the interaction with tasks is sequential, they are finite. Therefore, the agent might encounter the same bandit problem more than once and can leverage the previous experience. Besides, In the case of contextual bandits, a low dimensional structure \citep{cella2021multi,cella_multitask_2022,yang2020impact} or prior knowledge of the relations between tasks \citep{yang2020laplacian} provably reduces the regret
\par In this work, we borrow the concept of low dimensional structure from multi-task learning and leverage it with the concentration of tasks around some space region to improve the regret bound over a family of contextual linear bandits tasks. Indeed, assumptions such as high task concentration around a mean or strictly belonging to a low-dimensional subspace are restrictive. Thus, we aim at relaxing them. Finally, our approach is interpretable as learning an approximation of the covariance matrix of the tasks where the total variance is dominated by the contributions of a few principal components that span the subspace so it tightly relates to \citet{peleg2022Metalearning}; Nevertheless, one of our proposed algorithms does not rely on the prior update.

\section{Problem Formulation}
\label{sec:pb_formulation}
We consider an agent (learner, interchangeably) that sequentially interacts with several contextual bandit tasks. While learning one task over $n$ rounds, at each round $k$, the learner selects an arm $a_k$ from a dynamic set of arms $\mathcal{A}_k$ with associated context vector $\x_{a_k} \in \Rset^d$ satisfying $\norm{\x_{a_k}}\leq 1$. Then it receives a reward $r_k=\x_{a_k}^\top\thetabf^*+\epsilon_k$, where $\thetabf^* \in \Rset^d$ is the true task parameter to estimate. For different tasks, $\thetabf^*$ is independently drawn from a probability distribution $\rho$ over $\Rset^d$ (i.i.d.) with mean $\mubf$. Besides, they are bounded, formally, $\norm{\thetabf^*}\leq V$ for some $V>0$.\footnote{Throughout the paper, $\norm{\cdot}$ denotes the Euclidean norm.}~Moreover, $\epsilon_k$ is the zero-mean $1$-subgaussian noise such that $\{\epsilon_k\}_{k=1}^n$ are independent and identically distributed (i.i.d). 
\par Our main assumption is that the distribution $\rho$ has low variance along certain directions in space which ought to be learnt. \Cref{asm:main_assumption} states this requirement formally. Besides, an illustration of a sampling from such a task distribution in 3 dimensions appears in \Cref{figure:distr_scheme}. Finally, we denote the covariance of $\rho$ as $\Sigmabf$ with ordered eigenvalues $\sigma_1 \geq\cdots\geq\sigma_d$. 
\begin{assumption}
\label{asm:main_assumption}
    There exists an orthogonal projection matrix $\P\in\mathbb{R}^{d\times d}$ with rank $p$ such that:
    \begin{align*}
         \mathrm{Var}_{\rho} &\coloneqq \mathbb{E}_{\thetabf^*\sim \rho}\left[\norm{(\I-\P)\left(\thetabf^*-\mubf\right)}^2\right] \\
         &\ll \mathbb{E}_{\thetabf^*\sim \rho}\left[\norm{\P\left(\thetabf^*-\mubf\right)}^2\right]\\
         &\leq \mathrm{Var}_{\mathrm{max}} \coloneqq \mathbb{E}_{\thetabf^*\sim \rho}\left[\norm{\thetabf^*-\mubf}^2\right],
    \end{align*}
    \begin{equation*}
         \mathrm{Var}_{\rho} \ll \mathbb{E}_{\thetabf^*\sim \rho}\left[\norm{\P\thetabf^*}^2\right].
    \end{equation*}
\end{assumption}
\begin{figure}[ht!]
    \centering
    \includegraphics[scale=0.3
    ,]{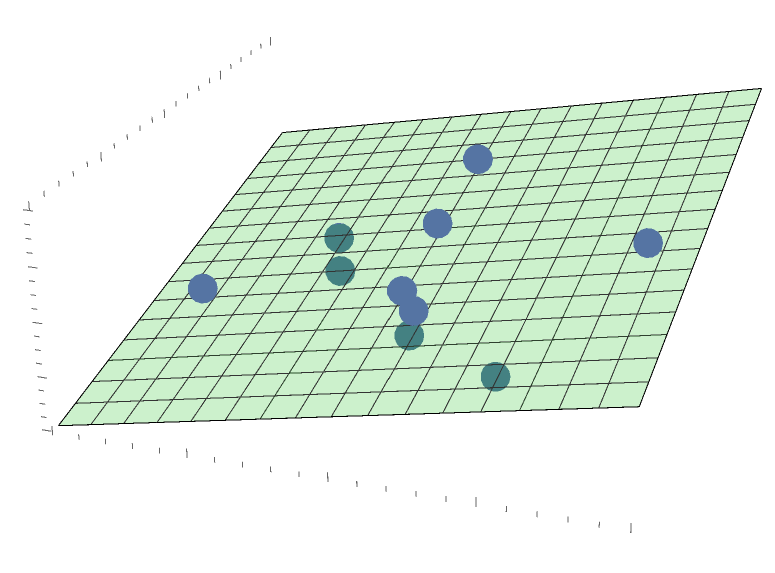}
    \caption{Sample of task parameters (blue points) from a distribution with low variance along one dimension.}
    \label{figure:distr_scheme}
\end{figure}
Our goal is to learn $\P$ and $\mubf$ as well as bound the expected transfer regret $\mathcal{R}(n)$ adapted from \citet{cella2020meta} defined as:
\begin{equation}
    \mathcal{R}(n) = \mathbb{E}_{\thetabf^*\sim\rho}\left[\mathbb{E}\left[\sum_{k=1}^n\left(\x_{a_k^*}-\x_{a_k}\right)^\top\thetabf^*\right]\right],
\end{equation}
where  $a_k^*\coloneqq\argmax_{a\in\mathcal{A}_k}\x_a^\top\thetabf^*$ is the optimal arm at round $k$.
   We propose two different approaches to exploit the knowledge of $\P$. First we present a variation of the standard LinUCB algorithm \citep{abbasi2011improved} by adjusting the regularization term in the regularized least squares optimization problem. Our second approach is a variation of the linear Thompson Sampling algorithm \citep{agrawal:linearTS}, where we adjust the covariance term of the normal distribution from which a task parameter is sampled from after every task according to the learned projection.
\section{Subspace Learning}
\label{sec:subspace_learning}
We use an online PCA version, namely, Candid Covariance-Free Incremental Principal Component Analysis (CCIPCA) \cite{online_pca}, to learn the underlying subspace structure from estimated task parameters. The core idea is to find an approximation of a set of orthonormal vectors that represent the principal components of a vertically concatenated data set $\Thetabf=[(\thetabf(i))-\Bar{\thetabf})^{\top}]_{i\in\{1,...,t\}}$, with $\Bar{\thetabf}\coloneqq\cfrac{1}{t}\sum_{i=1}^t\thetabf(i)$. Vector $\thetabf(i)$ denotes the $i$th task, which was estimated after a total of at least $n$ rounds. Upon finishing a task after $n$ rounds, the agent updates the learned projection matrix. Nevertheless, applying PCA is costly in the long run, whereas an online estimation mitigates the costs while offering sufficient estimations on the learned projection. Starting with a set of orthonormal eigenvectors $\{\u_1,...,\u_d\}$ and their corresponding eigenvalues $\{\sigma_1,...,\sigma_d\}$ based on the covariance matrix $\frac{1}{t}\Thetabf^\top \Thetabf$, we define $\v_j\coloneqq\sigma_j \u_j$ for $j\in\{1,...,d\}$ as the set of scaled principal components. Here we assume $\Bar{\thetabf}=\nul$, for the general case, the task parameters have to be centralized. Each additional task parameter $\thetabf(i)$ adjusts the estimation of $\v_j$ \emph{i.e.} after every round, every principal component $\v_j$ will be updated as
\begin{equation}
    \v_{j,i+1}=\frac{i}{i+1}\v_{j,i}+\frac{1}{i+1}\z_{i+1}\z^\top_{i+1}\frac{\v_{j,i}}{\norm{\v_{j,i}}},
    \label{eq:principal_components}
\end{equation}
with $\z_{i}$ determined to ensure orthogonality of the eigenvector estimations. Formally, to compute $\v_{j,i+1}$ we have:
\begin{equation*}
    \z_{i+1}=\thetabf(i+1)-\sum_{j'=1}^{j-1}\left(\thetabf^\top(i+1)\u_{j',i}\right)\u_{j',i}.
\end{equation*}
CCIPCA is especially beneficial as it is hyperparameter-free. Besides, it estimates the eigenvalues and the corresponding eigenvectors of all principal components. The eigenvalue estimations are essential when choosing the rank $p$ of the projection, which is generally unknown. The vectors $\u_i$ with the $p$ highest values $\sigma_i$ are selected as principal components. We define their horizontal concatenation as $\U=[\u_j]_{j\in\{1,...,p\}} \in \Rset^{d \times p}$.

\begin{remark}
    The choice of $p$ depends on the respective eigenvalues, a common choice would be to maximize the eigengap, thus $p=\argmax_{p'}\sigma_{p'}-\sigma_{p'+1}$.
\end{remark}

The projection matrix $\P$ with rank $p$ as well as the orthogonal projection $\P^{\perp}$ with rank $q=d-p$ can then be constructed using of the principal components as
\begin{equation}
    \P = \U\U^\top, \quad  \P^{\perp}=\I-\P
    \label{eq:construct_proj}
\end{equation}
We will use the learned projections to exploit the low dimensional subspace structure in both LinUCB and Thompson sampling setting.
\section{Projection Meta learning with LinUCB}
\label{sec:proj_linucb}
In this section, we present our contextual bandit algorithms based on LinUCB.
\subsection{Basics of LinUCB}
In classic LinUCB, at each round $k$, the agent uses the collection of previously selected actions $\D_k=[\x^\top_{a_i}]_{i\in\{0,...,k-1\}}$ and the corresponding rewards $\y_k=[r_i]_{i\in\{0,...,k-1\}}$ to estimate the task parameter $\thetabf_k$ by solving the following regularized least squares optimization problem:
\begin{equation}
\label{eq:theta}
    \thetabf_k=\argmin_{\thetabf \in \Rset^d} \norm{\D_k\thetabf-\y_k}^2+\lambda\norm{\thetabf}^2,
\end{equation}
with $\lambda>0$ being the regularization parameter. The solution of \eqref{eq:theta} is the \emph{ridge estimator}. Given that, the learner selects an action that maximizes the UCB index
\begin{equation}
    \mathrm{UCB}(a) = \x_a^\top\thetabf_k+\gamma_k \norm{\x_a}_{\A_k^{-1}}, \footnote{Throughout the paper, $\norm{\cdot}_{\A}$ denotes the weighted norm: $\norm{\x}_{\A}=\sqrt{\x^{\top}\A\x}$.}
    \label{eq:ucb}
\end{equation}
with $\A_{k} \coloneqq \lambda \I + \D^\top_k\D_k$, $\thetabf_k=\A_k^{-1}\D^\top_k\y_k$ and $\gamma_k>0$ as an upper bound on the confidence set radius proposed in \citet{abbasi2011improved}. The additional term scaling is essential for exploration as $\norm{\x}_{\A_k^{-1}}$ is maximized for context vectors that have the least correlation with already explored arms.
\subsection{LinUCB with Projection Bias}
In our first proposal, we enhance the LinUCB by including the knowledge of the projection matrix $\hat{\P}$. The agent learns $\hat{\P}$ by an online PCA algorithm using the parameters of $t$ already learned tasks. To enforce the knowledge of the affine subspace during learning, we formulate the following optimization problem for a given task, where we define $\hat{\thetabf}_k$ as the minimizer over $\thetabf\in\Rset^d$ in the following objective:
\begin{equation}
     \norm{\D_k\thetabf-\y_k}^2+\lambda_1\norm{\hat{\P}^{\perp}(\thetabf-\Bar{\thetabf})}^2 + \lambda_2 \norm{\hat{\P}\thetabf}^2,
    \label{eq:projecten_linucb_optimization}
\end{equation}
with $\hat{\P}^{\perp}\coloneqq\I-\hat{\P}$, $\lambda_1 > 0$ and $ \lambda_2 > 0$. Besides,
\\$\Bar{\thetabf}\coloneqq\cfrac{1}{t}\sum_{i=1}^t\thetabf(i)$ is the mean of the ridge regression estimators of the $t$ previous tasks. We justify the explicit choice of the regularization parameters in the analysis. Problem \eqref{eq:projecten_linucb_optimization} has a closed form solution given by
\begin{equation}
    \hat{\thetabf}_k = (\D_k^\top\D_k+\lambda_1\hat{\P}^{\perp}+ \lambda_2\hat{\P})^{-1}\left(\D_k^\top\y_k+\lambda_1\w\right),
    \label{eq:closedform}
\end{equation}
with $\w\coloneqq\hat{\P}^{\perp}\Bar{\thetabf}$. The second regularization term in \cref{eq:projecten_linucb_optimization} scaling with $\lambda_2$ is necessary so that our closed form solution (\ref{eq:closedform}) is well defined i.e., it enables us to determine the inverse of 
\begin{equation}
\B_k\coloneqq\D_k^\top\D_k+\lambda_1\hat{\P}^{\perp} + \lambda_2\hat{\P}^{\perp}. 
  \label{eq:Bdefinition}
\end{equation}
The case $\hat{\P}=\I$, which implies that all tasks are highly concentrated around the vector $\w=\Bar{\thetabf}$, would correspond to the setting of \cite{cella2020meta}.

Action selection is based on the principle of optimism in the face of uncertainty (OFUL), we propose an alternative UCB index by estimating the difference between mean reward $r$ and estimated reward $\hat{r}$:
\begin{align*}
    |\hat{r}-\mathbb{E}(r|\x)|&=|\x^\top(\hat{\thetabf}_k-\thetabf^*)|\\
    &\leq \norm{\hat{\thetabf}_k-\thetabf^*}_{\B_k}\norm{\x}_{\B_k^{-1}} \leq \gamma_k \norm{\x}_{\B_k^{-1}},
\end{align*}
with $\gamma_k\geq\norm{\hat{\thetabf}_k-\thetabf^*}_{\B_k}$. We provide an upper bound on the confidence set of the current estimation $\hat{\thetabf}_k$ in \Cref{sec:analysis}. The UCB function is then given by
\begin{equation}
    \mathrm{UCB}(a)=\x^\top_a\hat{\thetabf}_k+\gamma_k \norm{\x_a}_{\B_k^{-1}}
    \label{eq:proj_ucb}.
\end{equation}
%
\subsection{Analysis}\label{sec:analysis}
We start by providing a confidence set bound on the current estimation of our task parameter. We make use of an adapted concentration inequality provided by \cite{abbasi2011improved} in the following lemma.
\begin{lemma}[Self-normalized bound for vector-valued martingales]
    Let $\tau$ be a stopping time with respect to a filtration $\{\mathcal{F}_k\}_{k=1}^{\infty}$ and define $\etabf_k=\D^\top_k\epsilonbf$, with $\epsilonbf \in \Rset^k$ as subgaussian noise vector. Then, for every $\delta\in(0,1)$, with probability at least $1-\delta$ we have 
    \begin{equation*}
        \norm{\etabf_k}^2_{\B^{-1}_k}\leq \log(\frac{\det(\B_k)}{\delta^2\lambda_1^q\lambda_2^p}).
    \end{equation*}
    \label{lemma:self-concentrated-bound}
\end{lemma}
To emphasize the dimension of the subspace and the residual, in the lemma below, we bound the ratio of determinants in Lemma \ref{lemma:self-concentrated-bound}. 
\begin{lemma}
    Let $\lambda_1,\lambda_2 >0$ and $\B$ be defined as in \cref{eq:Bdefinition}. Then 
    \begin{align*}
        &\log(\frac{\det(\B_k)}{\det(\lambda_1\hat{\P}^{\perp}+\lambda_2\hat{\P})}) \leq \\
        &S_{k}^{\lambda_1,\lambda_2} \coloneqq p \log(1+\frac{k}{p\lambda_2})+q \log(1+\frac{k}{q\lambda_1}).
    \end{align*}
    \label{lemma:det_bound}
\end{lemma}

In the following lemma, we formulate the confidence set bound in our setting.
\begin{lemma}
    At round $k$, and with probability of at least $1-\delta$, the confidence set bound for $\hat{\thetabf}_k$ is given by
    \begin{align*}
        \norm{\hat{\thetabf}_k-\thetabf^*}_{\B_k} \leq \sqrt{S_{k}^{\lambda_1,\lambda_2}+\log(\frac{1}{\delta^2})} +\sqrt{\lambda_2}V+\frac{\lambda_1}{\sqrt{\lambda_2}}W,
    \label{confidence_set}
    \end{align*}
    where $W\coloneqq\norm{\hat{\P}^{\perp}(\thetabf^*-\Bar{\thetabf})}$. 
\label{lemma:confidence_set_bound}
\end{lemma}

Ideally, we want to show that the confidence set bound is tightened with the knowledge of the shared subspace and the corresponding projection matrix. That can be observed in the regularization terms scaling with $\norm{\hat{\P}^{\perp}\left(\Bar{\thetabf}-\thetabf^*\right)}$, which is small with high probability due to \Cref{asm:main_assumption}. The second regularization based term scales with $\norm{\hat{\P}\thetabf^*}$ and $\lambda_2$ and guarantees our problem to be well posed. In addition, the choice of $\lambda_2 \ll \lambda_1$ enforces leveraging our assumption on $\rho$.

Before establishing an upper bound for the expected transfer regret, we deliver an error estimation on our projection matrices. For this purpose, we use the eigengap $\Delta_{\sigma}\coloneqq\sigma_p-\sigma_{p+1}$, which is assumed to be positive, where $p$ is the dimension of the low dimensional subspace. A projection $\P$ depends on the number $p$ of selected eigenvectors, thus it can be assigned a specific eigengap. The following results shows the benefit of large eigengaps.
\begin{lemma}
    Let $\Bar{\thetabf}=\frac{1}{t}\sum_{i=1}^t\thetabf(i)$ be the empirical mean of $L_2$ regularized task parameter estimations $\thetabf(i)$ of true parameters $\thetabf^*(i) \sim \rho$. Assume that each $\thetabf(i)$ was estimated after the selection of at least $n$ arms. Let $\hat{\P}^{\perp}$ and $\Delta_{\sigma}>0$ be the estimation of $\P^\perp$ and the eigengap of $\Sigmabf$, respectively. We have
    \begin{align*}
&\mathbb{E}_{\thetabf^*\sim\rho}\left[\mathbb{E}\left[\norm{\hat{\P}^{\perp}\left(\Bar{\thetabf}-\thetabf^*\right)}\right]\right]= \\&\mathcal{O}\left(\sqrt{\mathrm{Var}_{\rho}+b^2\beta_{d}^2+\epsilon_{\mubf}^2+c\epsilon_{\Sigmabf}^2}\right),
    \end{align*}
    \label{projection_error}
with $\epsilon_{\mubf}=\frac{2\log(2t)}{t}+\sqrt{\frac{2\log(2t)\mathrm{Var}_{\rho}}{t}}$, $\epsilon_{\Sigmabf}^2=\frac{C\log(2t)}{t}$, $C$ is an absolute constant, $b=1+64\sqrt{2p}\frac{V^2}{\Delta_{\sigma}}$, $c=\frac{128 p V^2}{\Delta_{\sigma}^2}$, $\beta_{d}=\frac{1}{\sqrt{\lambda_{\min}}}\left[\sqrt{d\log(1+\frac{n^2V^2}{d})+2}+\sqrt{\frac{1}{n}}\right]$ and $\lambda_{\min}$ is the smallest of the minimal eigenvalues of matrices $\A_n$. 
\end{lemma}
In what follows, we define
\begin{equation*}
    Y\coloneqq\mathrm{Var}_{\rho}+\beta_{d}^2\left(1+64\sqrt{2p}\frac{V^2}{\Delta_{\sigma}}\right)^2+\epsilon_{\mubf}^2+\frac{128 p \epsilon_{\Sigmabf}^2V^2}{\Delta_{\sigma}^2}.
    \label{eq:Y}
\end{equation*}
The mean concentration error is $\epsilon_{\mubf}$ that converges to zero for a sufficiently large number of tasks. Besides, $\epsilon_{\Sigmabf}$ gives us the concentration error bound of the covariance estimated by the true task parameters $\thetabf^*$ and converges to zero. The bound depends heavily on the eigengap of the true covariance. Larger eigengaps reduce the expected error term and increase the reliability of the projection estimation. For the analysis, we assume $\Delta_{\sigma}>0$ for the chosen value of $p$. By assumption, $\mathrm{Var}_{\rho}$ is relatively small. Thus, the complete term is mostly dominated by $\beta_{d}^2$, which is an upper bound on the mean squared error (MSE) of the ridge estimator from the standard linUCB case. By selecting $\lambda\sim\frac{1}{n}$, the MSE of the ridge estimator converges to the estimator's variance. This variance, in turn, scales with the subgaussian noise term added on the rewards and also depends on the singular values of the respective data covariance matrix $\D^\top\D$. Ideally, we would prefer non-zero singular values. That implies that the set of context vectors yield information along any dimension, which would minimize the variance of the ridge estimator; Nevertheless, our setting does not guarantee this. 

We establish an upper bound on the transfer regret in the following theorem.
\begin{theorem}
     Assuming that $\P$ and $\mubf$ are known, the expected transfer regret of the projected LinUCB algorithm is upper bounded by
    \begin{align*}
        &\mathcal{R}(n)=\\
        &\mathcal{O}\left(\sqrt{n}\left(p\log(1+\frac{nV^2}{p})+q\log(1+\frac{n\sqrt{\mathrm{Var}_\rho}}{q})\right)\right).
    \end{align*}
    If the assumptions of Lemma \ref{projection_error} hold, the expected transfer regret is upper bounded by
    \begin{align*}
        &\mathcal{R}(n)=\\
        &\mathcal{O}\left(\sqrt{n}\left(p\log(1+\frac{nV^2}{p})+q\log(1+\frac{n\sqrt{Y}}{q})\right)\right).
    \end{align*}
    \label{theorem:proj_ucb_bound}
\end{theorem}
\begin{remark}
    The case $p=d,q=0$ yields the expected transfer regret $\mathcal{O}\left(\sqrt{n}d\log(1+\frac{nV^2}{d})\right)$, when no actual meta learning takes place and each task is learnt independently by the LinUCB algorithm.
\end{remark}
The results show that our approach is at most beneficial when $p$ is as low as possible such that \Cref{asm:main_assumption} still holds. In that case, increasing $\lambda_1$ in the algorithm reduces the overall regret bound, further supporting the argument that we made while discussing the confidence set bound in \Cref{lemma:confidence_set_bound}. By setting $\lambda_1=\frac{1}{\sqrt{Y}}$, the term $S^{\lambda_1,\lambda_2}_{n}$ defined in Lemma \ref{lemma:det_bound} changes such that only the $p$ dependent term becomes relevant as $Y$ significantly decreases and in turn $\log(1+\frac{n\sqrt{Y}}{q})$ as well,
essentially reducing the effective dimension of the problem to $p$ and indicating that less exploration is required within the $q$-dimensional subspace.
\section{Projection Meta Learning with Linear Thompson Sampling}
\label{sec:proj_TS}
%
\subsection{Basics of Linear Thompson Sampling}
LinUCB and linear Thompson sampling have the same requirements and assumptions concerning the linear relation between expected rewards and context vectors. Their difference lies in the decision-making process: In the former, the learner maximizes a UCB function by selecting the action at every round, whereas in the latter, it utilizes a Gaussian posterior calculated as $\mathcal{N}(\thetabf_k, v^2\A^{-1}_k)$, with $\thetabf_k$ estimated through solving the regularized least squares as done in LinUCB. From which, the learner then samples a parameter vector $\Tilde{\thetabf}_k$. It then selects the actions as
\begin{equation*}
    a = \argmax \x_a^\top \Tilde{\thetabf}_k.
\end{equation*}
The posterior is built from the prior of the previous instance given by $\mathcal{N}(\thetabf_{k-1}, v^2\A^{-1}_{k-1})$. This means that at $k=1$, during the initialization, we have $\A_0=\I$. 
The sampling process reflects the uncertainty of the current estimation $\thetabf_k$ and directly indicates the exploration behaviour of the learner. A low variance across a specified dimension indicates a high confidence of the current estimation and vice versa. Thus during initialization with $\A_0=\I$, there is equal exploration potential along any direction.
\subsection{Thompson Sampling with Linear Payoffs within an Affine Subspace}
Our second proposal is a variation of the linear Thompson sampling: We change the posterior from which $\Tilde{\thetabf}$ is sampled. The mean of the new distribution is the biased regularization solution $\hat{\thetabf}$ of \cref{eq:projecten_linucb_optimization}, and its covariance matrix is  $\B^{-1}$. Thus,
\begin{equation}
\label{eq:SecondModel}
    \Tilde{\thetabf}\sim\mathcal{N}(\hat{\thetabf}, v^2\B^{-1}).   
\end{equation}
In \cref{eq:SecondModel}, $v$ is a hyper-parameter that we determine in the analysis. During initialization, we have $\B_0=\lambda_1\hat{\P}^{\perp}+\lambda_2\hat{\P}$ and its inverse as covariance for the prior distribution. By choosing $\lambda_1\gg\lambda_2$ similar to the projected LinUCB setting, we embed our knowledge of the affine subspace into the prior. That way, the sampling process of $\Tilde{\thetabf}$ incorporates the low variance along the orthogonal subspace.
\subsection{Analysis}
The analysis in inspired by \cite{agrawal:linearTS}. First, we define the following two events:
\begin{definition}
    The event $E_{r}$ occurs if 
    \begin{equation*}
        \forall a \in \mathcal{A}_{k}: \left|\x_a^\top\left(\hat{\thetabf}_k-\thetabf^*\right)\right|\leq l_n\norm{\x_a}_{\B^{-1}_{k}},
    \end{equation*}
    with $l_n\coloneqq\sqrt{2\log(\frac{1}{\delta}) (d+2)\log(n)+2K^2}$ and\\
    $K\coloneqq\frac{\lambda_2}{\sqrt{\lambda_1}}W+\sqrt{\lambda_1}V$.\\
    The event $E_{\thetabf}$ occurs if 
    \begin{equation*}
        \forall a \in \mathcal{A}_k: \left|\x_a^\top\left(\hat{\thetabf}_k-\Tilde{\thetabf}_k\right)\right|\leq \sqrt{2d+6\log(n)}v\norm{\x_a}_{\B^{-1}_{k}},
    \end{equation*}
    with $v\coloneqq4\sqrt{\log(\frac{1}{\delta})\frac{d+2}{\alpha}}$ and $\alpha\in(0,1)$.
    \label{def:event_definitions}
\end{definition}
The event $E_r$ essentially reflects the confidence set bound discussed for the projected LinUCB case. It gives the probability that our current estimation $\hat{\thetabf}K$ lies within the bound. The event $E_{\thetabf}$ is directly linked to the sampling procedure of $\Tilde{\thetabf}_k$. It gives the probability that the reward estimation of the sampled $\Tilde{\thetabf}_k$ is within some limited range of the estimated reward of $\hat{\thetabf}_k$. 
 Below, we define a filtration, containing all necessary information for the algorithm.
\begin{definition}
    We define the filtration $\{\mathcal{F}_{k}\}_{k\in\{1,...,n\}}$
    with sub-$\sigma$-algebras $\mathcal{F}_{k-1}$ at round $k$ generated by the current action set and the history up to round $k-1$: $\mathcal{F}_{k-1}=\{\mathcal{A}_k,\mathcal{H}_{k-1}\}$, with the history being recursively defined as: 
    \begin{equation*}
        \mathcal{H}_{k}=\{\mathcal{A}_{k},\hat{\thetabf}_{k},\B_{k},\norm{\x_{a}}_{\B^{-1}_{k}},\mathcal{N}(\hat{\thetabf}_{k}, \B^{-1}_{k})\} \cup\mathcal{H}_{k-1}.
    \end{equation*}
\end{definition}
The next lemma states the probability of the events $E_{r}$ and $E_{\thetabf}$.
\begin{lemma}
    For all $\delta \in (0, 1)$, the probability of event $E_r$ is bounded from below as follows: $\mathrm{Pr}(E_r) \geq 1-\frac{\delta}{n^2}$. Moreover, for all possible filtrations $\mathcal{F}_{k-1}$, the probability of event $E_{\thetabf}$ is bounded from below as follows: $\mathrm{Pr}(E_{\thetabf}|\mathcal{F}_{k-1})\geq 1-\frac{1}{n^2}$.
    \label{lemma:Event_bound}
\end{lemma}

In the following theorem, we establish an upper-bound for the transfer regret of the projected Thompson sampling algorithm%
\begin{theorem}
    The expected transfer regret of the projected Thompson sampling algorithm verifies
    \begin{equation*}
        \mathcal{R}(n)= \mathcal{O}\biggl(\left(d^{\frac{3}{2}}\log(n)+\sqrt{d}\log(n)^2\right)\sqrt{n S_{n}^{\frac{1}{\sqrt{Y}},\frac{1}{V^2}}}\biggr).
    \end{equation*}
    \label{theorem:proj_ts_regret}
\end{theorem}
\begin{remark}
    With $p=d,q=0$, the meta learning does not take place, i.e., the agent learns each task independently by the linear TS algorithm. As such, the expected transfer regret yields $\mathcal{O}\left(\left(d^{2}\log(n)+d\log(n)^2\right)\sqrt{n\log(1+\frac{nV^2}{d})}\right)$. 
\end{remark}
The results shows a the dependency on the dimensions $p$ and $q$, and the variance related term $Y$. For a sufficiently small $Y$, the terms scaling with $p$ would dominate the regret, so we expect greater improvements with decreasing $p$. The term scaling with $q$ would benefit from the low variance within the respective subspace. As suggested in \cite{agrawal:linearTS}, we chose $\alpha=\frac{1}{\log(n)}$ in the proofs.
\section{Algorithms}
The projected LinUCB and projected TS algorithms share many steps. Thus, we unify them and use sub-procedures.
\begin{algorithm}[ht]
\caption{Projected LinUCB/Thompson Sampling}
Initialize: $v > 0$, $\delta\in(0,1)$, $\lambda_1>\lambda_2>0$, $\lambda>0$, $\delta\in (0,1)$;\\
\For{$t\in\{1,...,T\}$}{
Initialize: $\A_0=\lambda\I$, $\b'_0=\nul$;\\
    Sample new task $\thetabf^*\sim\rho$;\\
     \If{$t<d$}{$\hat{\P}=\I$, $\hat{\P}^{\perp}=\nul$, $\w=\nul$;\\}
    \Else{Determine principal components and calculate $\hat{\P}$ and $\hat{\P}^{\perp}$ with $[\thetabf(i)]_{i\in\{1,...,t\}}$ according to \cref{eq:principal_components,eq:construct_proj} and $\w=\frac{1}{t-1}\hat{\P}^{\perp}\sum_{i=1}^{t-1}\thetabf(i)$;}
Initialize $\B_0=\lambda_1\hat{\P}^{\perp}+\lambda_2\hat{\P}$,  $\b_0=\lambda_1\hat{\P}^{\perp}\w$, $\hat{\thetabf}_0=\B_0^{-1}\b_0$;\\    
\For{$k\in\{0,...,n-1\}$}{
    Select arm $a_k$ according to respective arm selection strategy (\Cref{alg:lin_ucb_select} or \ref{alg:TS_select});\\
    Collect immediate reward $r_k$;\\
    $\B_{k+1}=\B_k+\x_{a_k}\x_{a_k}^\top$;\\
    $\A_{k+1}=\A_k+\x_{a_k}\x_{a_k}^\top$;\\
    $\b_{k+1}=\b_k+r_k\x_{a_k}$;\\
    $\b'_{k+1}=\b'_k+r_k\x_{a_k}$;\\
    $\hat{\thetabf}_{{k+1}}={\B}^{-1}_{k+1}\b_{k+1}$;}
    $\thetabf(t)=\A^{-1}_{n}\b'_{n}$;}
\end{algorithm}
\begin{algorithm}[t]
    \caption{Projected LinUCB Arm Selection Routine}
    \label{alg:lin_ucb_select}
    Input: $\hat{\thetabf}_k$, $\B_k$;\\
    $\gamma_k=\log(\frac{\det(\B_k)}{\delta^2\lambda_1^q\lambda_2^p})+\sqrt{\lambda_2}V+\frac{\lambda_1}{\sqrt{\lambda_2}}W$;\\
     Select arm $a_k=\argmax_a\mathrm{UCB}(a)$ from \eqref{eq:proj_ucb};
\end{algorithm}
\begin{algorithm}[ht!]
    \caption{Projected TS arm selection routine}
    \label{alg:TS_select}
    Input: $\hat{\thetabf}_k$, $\B_k$, $\alpha\in (0,1)$;\\
    $v=4\sqrt{\log(\frac{1}{\delta})\frac{d+2}{\alpha}}$;\\
    sample $\Tilde{\thetabf}_k\sim \mathcal{N}(\hat{\thetabf}_k, v^2\B_k^{-1})$;\\
    Select arm $a_k=\argmax_a\x_a^{\top}\Tilde{\thetabf}_k$;
\end{algorithm}
We introduce an initialization phase for learning the subspace, as it may only be well-defined after including sufficient task parameters. Enforcing the subspace learning already from the first task might lead to zero-dimensional subspace with $\hat{\P}=\nul$ that would degrade the overall performance.
In the projected LinUCB algorithm, we require the estimation of $\gamma_k$ taken from \Cref{lemma:det_bound}, which in turn requires the value of $W=\norm{\hat{\P}^{\perp}\left(\thetabf^*-\Bar{\thetabf}\right)}$, which is intractable but we work around this issue by using the simple bound $W \leq 2V$. Since $\gamma_k$ acts more as exploration scaling factor, we do not lose any benefit from the meta learning as the actual knowledge transfer becomes relevant in the calculations of $\B_k$ and $\hat{\thetabf}_k$ 
\section{Numerical Experiments}
\label{sec:experiments}
We test our algorithms experimentally on synthetic data and on real world data taken from the MovieLens data set.
\subsection{Synthetic Data Experiments}
We sampled the context vectors from a zero mean normal distribution with a diagonal covariance matrix whose elements followed a uniform distribution. Following \cite{mezzadri2006generate}, we used a randomly generated orthogonal matrix to define a subspace. We project the randomly generated task parameters onto the subspace and add a multivariate Gaussian noise term in the orthogonal direction to the given subspace to simulate the variance of the task distribution. One drawback of this approach is that it misses the benefits of subspace learning during the first tasks. That is because a subspace with dimension $p$ that ought to be learned requires at least $q=d-p$ data points or task parameters to use the PCA algorithms successfully. Thus, we also implement an initialization phase to prevent subspace learning until learning at least $d$ task parameters. Note that we require at least $d$ tasks as we do not use our knowledge of $p$. We consider a task as finished after at least $n=250$ rounds. \\
\Cref{fig:synthetic_expected_plot} shows the expected transfer regret for $d=30$ and $p=15$, with the oracle and the algorithms of \cite{cella2020meta} (B-OFUL), \cite{peleg2022Metalearning} (M-TS) as benchmarks. The projected Thompson sampling (P-TS) approach performs as well as projected LinUCB (P-LinUCB), while the oracle using the true projection and mean is the most efficient one. Our algorithms are significant improvements the other baselines. The superiority of our approach mainly stems from its generality compared to \citet{cella2020meta}. The algorithm provided by \citet{peleg2022Metalearning} has the worst performance mainly due to the regret contributed by the required forced exploration within a task. Additionally, \Cref{fig:synthetic_meta_plot} shows the total cumulative regret over tasks, which does not suffer from overlapping error bars as they become negligible.
\begin{figure}[htb]
    \vspace{.3in}
    \centering
    \begin{subfigure}{0.23\textwidth}
    \centering
    \includegraphics[width=\textwidth]{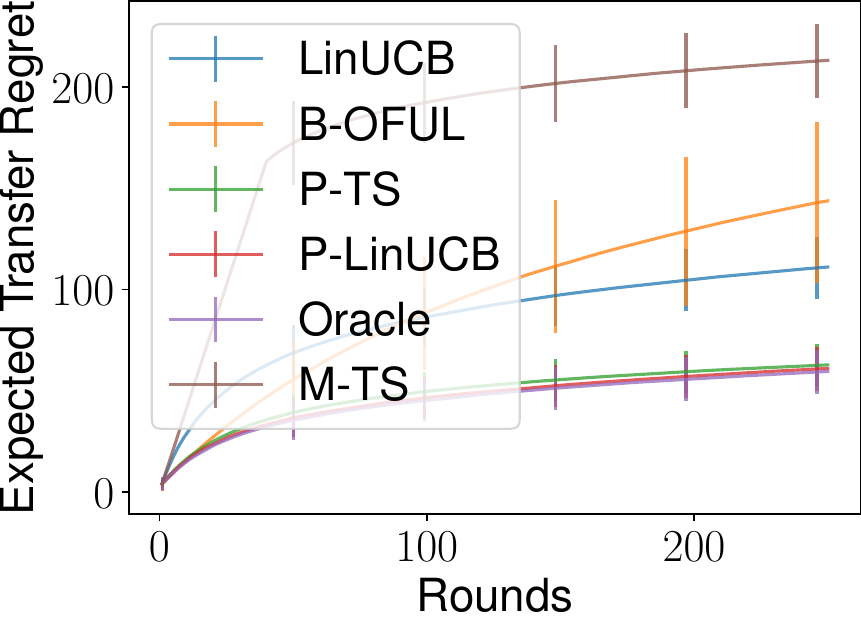}
    \caption{Synthetic data plots of the expected transfer regret as a function of number of rounds.}
    \label{fig:synthetic_expected_plot}
    \end{subfigure}
    \hfill
    \begin{subfigure}{0.23\textwidth}
    \centering
    \includegraphics[width=\textwidth]{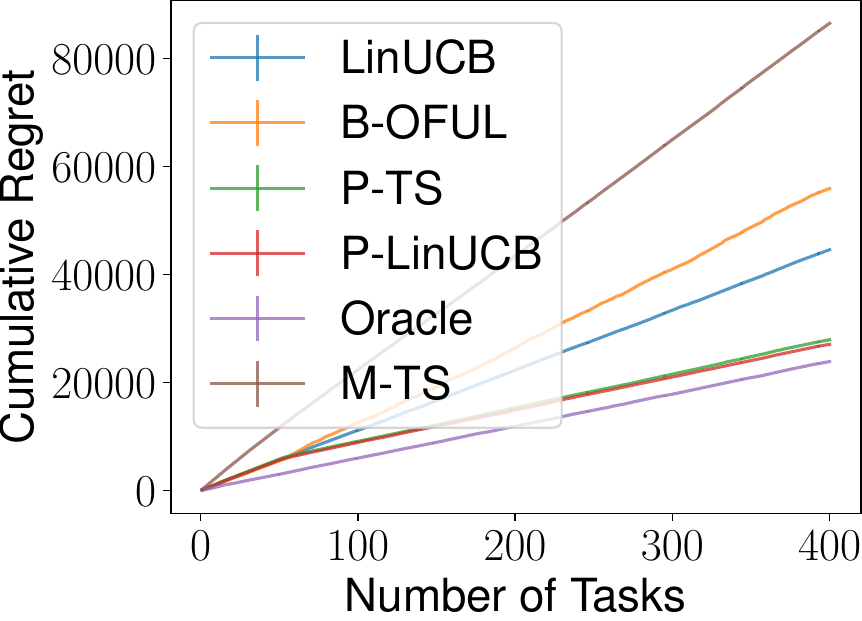}
    \caption{Synthetic data plots of the cumulative regret over the number of tasks.}
    \label{fig:synthetic_meta_plot}
    \end{subfigure}
    \caption{Expected transfer and total cumulative regret plots of the LinUCB and Thompson sampling methods compared to their projection counterparts and additional baselines.}
\end{figure}
To further emphasize the benefit of exploiting the knowledge on any dimensional low variance, \Cref{fig:dimensional_regret} shows the total accumulated regret of the projected LinUCB algorithm after $T$ tasks with $n$ rounds of learning each as a function of $q=\rank(\hat{\P}^{\perp})$. Note that at $q=0$, the plot shows the total regret using classic LinUCB. As expected, the regret reaches its minimum when $q$ is equal to the rank of the true projection, which is $q=15$ in this case. Nevertheless even for different values of $q$, there is a clear benefit over the classic approach. In \Cref{fig:W_plot} we plot $\left|\mathbb{E}_{\thetabf^*\sim\rho}[\mathbb{E}[W]]^2/\mathrm{Var}_{\rho}-1\right|$ as a function of number of tasks. As expected the curves for P-LinUCB and P-TS imply that $\mathbb{E}_{\thetabf^*\sim\rho}[\mathbb{E}[W]]^2$ is close to $\mathrm{Var}_{\rho}$. The curve for the B-OFUL algorithm assumes $\P^{\perp}=\I$, disregarding the covariance and thus resulting into higher values. Note that lower values imply greater transfer in between sequential tasks as the projection matrix would be well estimated.
\begin{figure}
     \vspace{.3in}
    \centering
    \begin{subfigure}{0.23\textwidth}
    \centering
    \includegraphics[width=\textwidth]{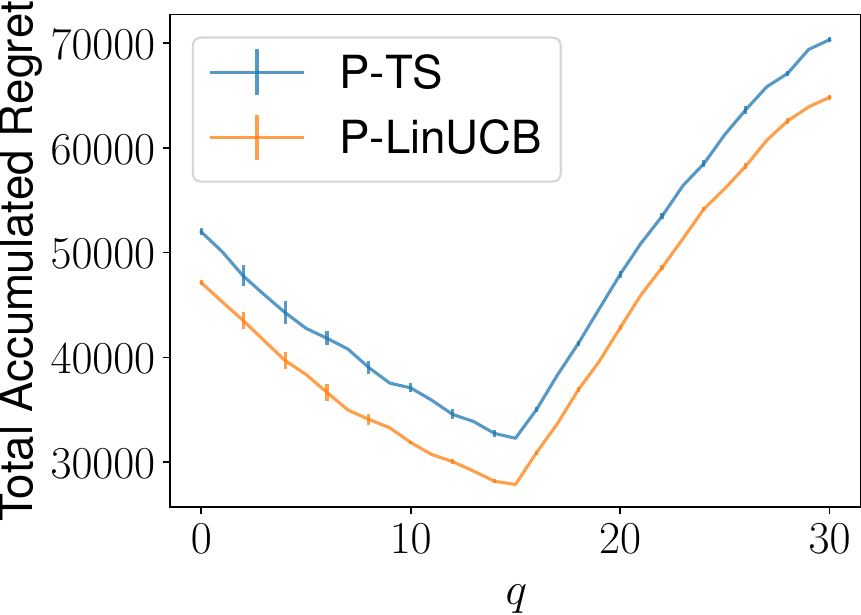}
    \caption{Total accumulated regret after 400 tasks as a function of $\rank(\hat{\P}^{\perp})$.\\}
    \label{fig:dimensional_regret}
    \end{subfigure}
    \hfill
    \begin{subfigure}{0.23\textwidth}
    \centering
    \includegraphics[width=\textwidth]{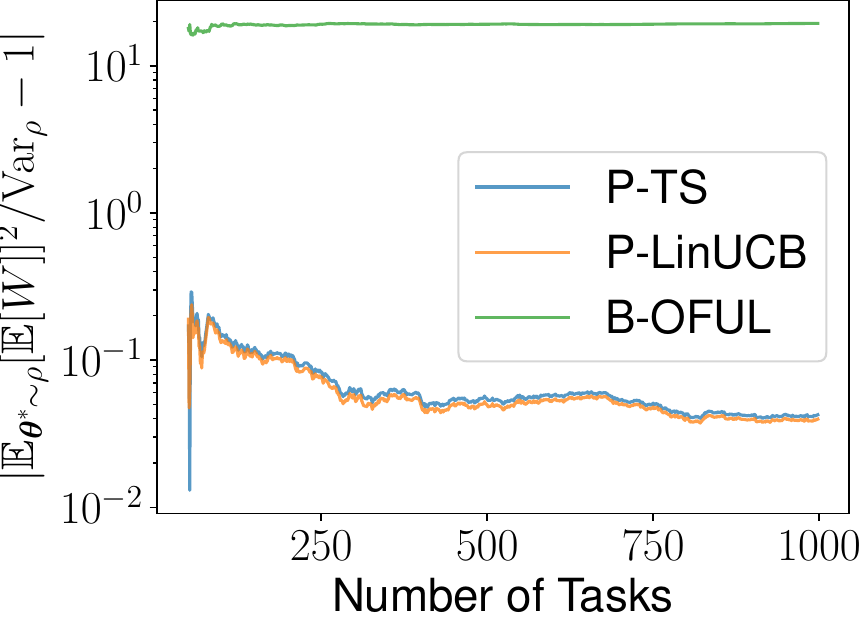}
    \caption{Relative error of $\mathbb{E}_{\thetabf^*\sim\rho}[\mathbb{E}[W]]^2$ as a function of number of tasks in a logarithmic scale.}
    \label{fig:W_plot}
    \end{subfigure}
    \caption{}
\end{figure}
\subsection{Real Data Experiments}
We use MovieLens data to test our algorithms in a real-world environment. MovieLens data contains information about over 6000 users that represent the tasks in our setting. Besides, it includes over 3000 movies, which are the arms with their corresponding context vectors. The context vectors are 18-dimensional, each denoting a possible genre. If a movie has a label for a specific genre, the corresponding entry for that genre in the context vector is 1. With at most six different genres assigned to a single movie, we normalize the context vectors such that we have $\norm{\x_a}\leq 1$. Each movie has some available ratings between 1 and 5, given by a user who has watched that movie. Each rating represents a reward for our algorithm. We normalize all such ratings so that $r\in[0,1]$. We further process the data by grouping the users by their profession or gender and run the algorithm within that set of users. That method stems from the assumption that groups of similar users might share an affine subspace. For every user (task), we run the algorithms for at least $n=250$ rounds. 
\begin{figure}
    \vspace{.3in}
    \centering
    \begin{subfigure}{0.23\textwidth}
    \centering
    \includegraphics[width=\textwidth]{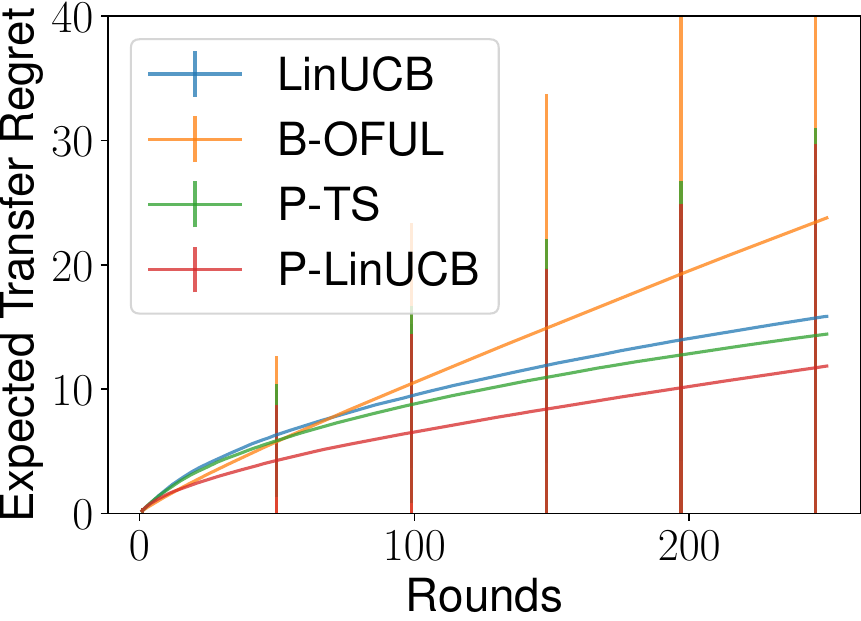}
    \caption{Real data plots of the expected transfer regret as a function of number of rounds.}
    \label{fig:real_expected_plot}
    \end{subfigure}
    \hfill
    \begin{subfigure}{0.23\textwidth}
    \centering
    \includegraphics[width=\textwidth]{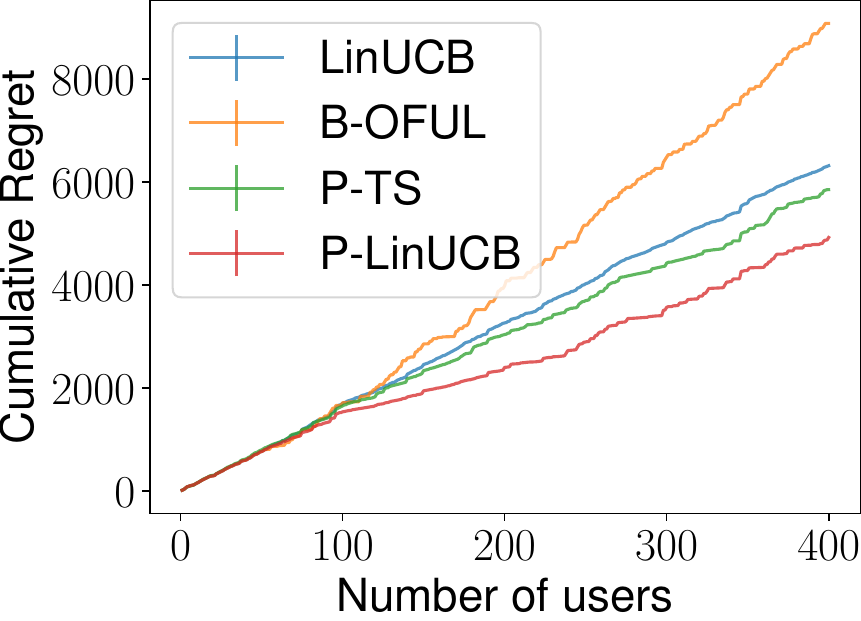}
    \caption{Real data plots of the cumulative regret over tasks/users.\\}
    \label{fig:real_meta_plot}
    \end{subfigure}
    \caption{Expected transfer regret and total regret plots of our algorithms and baselines applied to the MovieLens data set. We have included 400 users in the simulations.}
\end{figure}
We do not include the algorithm developed in \citet{peleg2022Metalearning} as baseline for experimentation using the real data set as it requires contexts from a distribution with an invertible covariance. The reason is that the authors do not use a regularizer on the minimum least squares solution for $\thetabf$, and thus find the inverse of $\D^\top\D$ after every finished task. However, in the MovieLens data set, that condition does not hold for many users, making the estimation of $\thetabf$ ill-posed.
In the real data experiments, we observe significant improvements of our models over the baselines in \Cref{fig:real_expected_plot} showcasing the expected transfer regret, and in \Cref{fig:real_meta_plot}, showcasing the total cumulative regret over users, which does not suffer from overlapping error bars. A significant point is that we did not perform data preprocessing besides normalizing the rewards and dividing the users into male and female. That would also explain the performance gap to the algorithm of \citet{cella2020meta}, as our assumption is more general and widely applicable.
\section{Discussion and Outlook}
Our work shows that obtaining knowledge about the underlying subspace structure in a meta-learning setting improves sequential task learning. More precisely, assuming a low variance along certain dimensions in the task distribution, we proposed two decision-making policies that exploit the knowledge of the subspace structure for sequential arm selection and significantly improve the performance of widely used algorithms, namely LinUCB and linear Thompson sampling. We provided an improved regret bound that manifests the dependency on the lower dimension, the low variance term, and the eigengap at the considered low dimension. We evaluated our methods numerically through experimentations on synthetic and real-world datasets, confirming their better performance than traditional benchmarks. The results are significant in the real data environments as the rewards do not necessarily follow a linear relation.\\
Possible extensions of this work include further generalization of our model by learning the variance of the task distribution along all dimensions. Another direction is to generalize our methods to non-linear settings, \emph{i.e}., when tasks concentrate around a low dimensional manifold.
\newpage
\subsubsection*{Acknowledgements}
This work was supported by Grant 01IS20051 from the German Federal Ministry of Education and Research (BMBF).
The authors thank the International Max Planck
Research School for Intelligent Systems (IMPRS-IS) for supporting Steven Bilaj.

\bibliography{aistats.bib}

\onecolumn
\appendix
\section{Notations}
\begin{table}[htb]
    \centering
    \begin{tabular}{l|l}
        Notation & Meaning \\
        \hline 
        $a$, $a^*$ & Arm and optimal arm yielding highest mean reward respectively\\
        $\mathcal{A}_k$ & Set of available arms\\
        $\x_a$ & Context vector associated with arm $a$ \\
        $d$ & Dimension of the context vectors\\
        $n$ & Horizon \\
        $r_k$ & Immediate reward at round $k$\\
        $\epsilon_k$ & Subgaussian noise added to the reward at round $k$\\
        $\D_k$ & Vertical concatenation of up to round $k$ collected context vectors $\x_{a}^{\top}$\\
        $\y_k$ & Concatenation of up to round $k$ collected rewards\\
        $\gamma_k$ & Upper confidence set bound of LinUCB or projected LinUCB algorithm in round $k$\\
        $v$ & Scaling factor for the covariance of the Thompson Sampling posterior\\
        $\lambda$, $\lambda_1$, $\lambda_2$ & Regularization parameters for ridge and projection based estimators\\
        $\rho$ & Task distribution\\
        $\Sigmabf$ & True covariance of $\rho$\\
        $\{\sigma_j\}_{j\in\{1,...,d\}}$ & Eigenvalues of $\Sigmabf$ \\
        $\Delta \sigma$ & Eigengap of $\Sigmabf$\\
        $\P$, $\hat{\P}$ & True subspace projection and its estimation respectively\\
        $\P^{\perp}$, $\hat{\P}^{\perp}$ & $\I-\P$ and $\I-\hat{\P}$ respectively \\
        $S_{k}^{\lambda_1,\lambda_2}$ & $p \log(1+\frac{k}{p\lambda_2})+q \log(1+\frac{k}{q\lambda_1})$ \\
        $\thetabf^*$, $\thetabf$, $\hat{\thetabf}$ & True task parameter, its ridge estimator and projections based estimator respectively \\
        $\Bar{\thetabf}$ & Mean of $t$ collected ridge estimations of true task parameters: $\frac{1}{t}\sum_i^t\thetabf(i)$ \\
        $p$ & Rank of $\hat{\P}$\\
        $q$ & Rank of $\hat{\P}^{\perp}$\\
        $\w$ & $\hat{\P}^{\perp}\Bar{\thetabf}$\\
        $\A_k$ & $\lambda\I+\D_k^{\top}\D_k$\\
        $\B_k$ & $\lambda_1\hat{\P}^{\perp}+\lambda_2\hat{\P}+\D_k^{\top}\D_k$\\
        $\b_k'$ & $\D_k^{\top}\y_k$\\
        $\b_k$ & $\D_k^{\top}\y_k+\lambda_1\hat{\P}^{\perp}\Bar{\thetabf}$ \\
        $V$ & Upper Bound on the norm of any true task parameter \\
        $W$ & $\norm{\hat{\P}^{\perp}(\thetabf^*-\Bar{\thetabf})}$\\
        $Y$ & Upper bound of $\mathbb{E}_{\thetabf^*\sim\rho}\left[\mathbb{E}\left[\norm{\hat{\P}^{\perp}\left(\Bar{\thetabf}-\thetabf^*\right)}\right]\right]^2$ \\
        $\alpha$ & Hyper parameter of Thompson Sampling algorithm\\
        $l_n$ &  $\sqrt{2\log(\frac{1}{\delta}) (d+2)\log(n)+2K^2}$\\
        $g_n$ & $\sqrt{2d+6\log(n)}v+l_n$\\
        $\mathcal{R}$ & Expected transfer regret \\
        $\norm{\cdot}$ & Euclidean norm \\
        $\norm{\cdot}_{\A}$ & Weighted norm: $\norm{\x}_{\A}=\sqrt{\x^{\top}\A\x}$
    \end{tabular}
    \caption{Table of Notations}
    \label{tab:notations}
\end{table}
\newpage
\section{Proof of \texorpdfstring{\Cref{theorem:proj_ucb_bound}}{}}
In order to prove \Cref{theorem:proj_ucb_bound} in the main paper we provide proofs of additional Lemmas here or refer to the original works:

\begin{proof}[Proof of Lemma \ref{lemma:self-concentrated-bound}]
Given Lemma 9 of \cite{abbasi2011improved}, we have:

\begin{equation}
    \norm{\etabf_k}^2_{\B^{-1}_k}\leq \log(\frac{\det(\B_k)}{\delta^2\det(\lambda_1\hat{\P}^{\perp}+\lambda_2\hat{\P})}),
\end{equation}

where the term $\det(\lambda_1\hat{\P}^{\perp}+\lambda_2\hat{\P})$ can be further evaluated knowing the eigenvalues of the matrix $\lambda_1\hat{\P}^{\perp}+\lambda_2\hat{\P}$. With orthogonal projections $\hat{\P}$ and $\hat{\P}^{\perp}$ and $\hat{\P}^{\perp}=\I-\hat{\P}$ it holds that for any eigenvector $\e_{\P}$ of $\hat{\P}$ we have: $\hat{\P}^{\perp}\e_{\P}=(\I-\hat{\P})\e_{\P}=\nul$ and vice versa for any eigenvector $\e_{\P^{\perp}}$ of $\hat{\P}^{\perp}$: $\hat{\P}\e_{\P^{\perp}}=\nul$. Thus any eigenvector of $\hat{\P}$ or $\hat{\P}^{\perp}$ is also an eigenvector of $\lambda_1\hat{\P}^{\perp}+\lambda_2\hat{\P}$:

\begin{align*}
    (\lambda_1\hat{\P}^{\perp}+\lambda_2\hat{\P})\e_{\P}&=(0+\lambda_2)\e_{\P},\\(\lambda_1\hat{\P}^{\perp}+\lambda_2\hat{\P})\e_{\P^{\perp}}&=(\lambda_1+0)\e_{\P^{\perp}},
\end{align*}

with eigenvalues $\lambda_1$ and $\lambda_2$. Lastly we require the multiplicities of both eigenvalues given by the dimension of nullspaces of the matrices $\lambda_1\hat{\P}^{\perp}+\lambda_2\hat{\P}-\lambda_1\I=(\lambda_2-\lambda_1)\hat{\P}$ for $\lambda_1$ and $\lambda_1\hat{\P}^{\perp}+\lambda_2\hat{\P}-\lambda_2\I=(\lambda_1-\lambda_2)\hat{\P}^{\perp}$ for $\lambda_2$, which are $q=\rank(\hat{\P}^{\perp})$ and $p=\rank(\hat{\P})$ respectively. Thus we get:
\begin{equation}
    \det(\lambda_1\hat{\P}^{\perp}+\lambda_2\hat{\P})=\lambda_1^q\lambda_2^p,
\end{equation}
finalizing our proof.
\end{proof}

\begin{proof}[Proof of Lemma \ref{lemma:det_bound}]
    Let $\lambda'_i$ be the singular values of $\D^\top\D$ and $\norm{\x_a}\leq 1$ then we have:
    \begin{align*}
        \log\left(\frac{\det(\B)}{\det(\lambda_1\hat{\P}^{\perp}+\lambda_2 \hat{\P})}\right)&\leq \sum_{i=1}^p\log(1+\frac{\lambda'_i}{\lambda_2})+\sum_{i=p+1}^d\log(1+\frac{\lambda'_i}{\lambda_1})\\
        &\leq p\log(1+\frac{1}{p\lambda_2}\sum_{i=1}^p\lambda'_i)+q\log(1+\frac{1}{q\lambda_1}\sum_{i=p+1}^d\lambda'_i)\\
        &\leq p\log(1+\frac{k}{p\lambda_2})+q\log(1+\frac{k}{q\lambda_1})
    \end{align*}
    where we applied the Jensen inequality in the second inequality and bounded the trace by $k\norm{\x_{a_k}}^2\leq k$ in the last inequality.
\end{proof}

\begin{proof}[Proof of Lemma \ref{lemma:confidence_set_bound}]
        We leave out the subscript $k$ during the proof for readability purposes.
        Our estimation of $\thetabf^*$ for the projected LinUCB algorithm yields:
           \begin{equation}
                \hat{\thetabf} = \B^{-1}\left(\D^\top\y+\lambda_1\w\right),
            \end{equation}
            thus we can write:
            \begin{align*}
                \norm{\hat{\thetabf}-\thetabf^*}_{\B} &= \norm{\B^{-1}\left(\D^\top\y+\lambda_1\w\right)-\thetabf^*}_{\B} \\
                &= \norm{\B^{-1}\left(\D^\top(\D\thetabf^*+\epsilonbf)+\lambda_1\w\right)-\thetabf^*}_{\B} \\
                &= \norm{\B^{-1}\left(\D^\top\epsilonbf+\lambda_1\w\right) - \B^{-1}\left(\lambda_1\hat{\P}^{\perp}+\lambda_2\hat{\P}\right)\thetabf^*}_{\B} \\
                &= \norm{\D^\top\epsilonbf+\lambda_1\w-\left(\lambda_1\hat{\P}^{\perp}+\lambda_2\hat{\P}\right)\thetabf^*}_{\B^{-1}} \\
                &\leq \norm{\D^\top\epsilonbf}_{\B^{-1}} + \lambda_1\norm{\hat{\P}^{\perp}\left(\Bar{\thetabf}-\thetabf^*\right)}_{\B^{-1}} + \lambda_2\norm{\hat{\P}\thetabf^*}_{\B^{-1}} \\
                &\leq \sqrt{2\log\left(\frac{\sqrt{\det(\B)}}{\delta\sqrt{\det(\lambda_1\hat{\P}^{\perp}+\lambda_2\hat{\P})}}\right)} + \frac{\lambda_1}{\lambda_{\min}(\B)}\norm{\hat{\P}^{\perp}\left(\Bar{\thetabf}-\thetabf^*\right)} + \frac{\lambda_2}{\lambda_{\min}(\B)}\norm{\hat{\P}\thetabf^*} \\
                &\leq \sqrt{2\log\left(\frac{\sqrt{\det(\B)}}{\delta\sqrt{\det(\lambda_1\hat{\P}^{\perp}+\lambda_2\hat{\P})}}\right)} +\sqrt{\lambda_2}V+\frac{\lambda_1}{\sqrt{\lambda_2}}W \\
                &\leq \sqrt{p\log(1+\frac{k}{p\lambda_2})+q\log(1+\frac{k}{q\lambda_1})+\log(\frac{1}{\delta^2})} +\sqrt{\lambda_2}V+\frac{\lambda_1}{\sqrt{\lambda_2}}W,
            \end{align*}
            
            where we used Lemma \ref{lemma:self-concentrated-bound} in the second inequality. Here, $\lambda_{\min}()$ is a function returning the minimal eigenvalue of a given matrix.
    \end{proof}

For the upper bound on the projection based error term in Lemma \ref{projection_error}, need to make some definitions: We denote $\mubf$ as the true mean of the distribution of tasks $\rho$, $\Bar{\thetabf}^*=\frac{1}{t}\sum_{i=1}^t\thetabf^*(i)$ as the mean estimated by the true task parameters and $\Bar{\thetabf}=\frac{1}{t}\sum_{i=1}^t\thetabf(i)$ as the mean estimated by the $L2-$regularized ridge estimators. We define $\sigma_1\geq\sigma_2\geq...\geq\sigma_d$ as ordered eigenvalues of the true covariance matrix $\Sigmabf$ and $\Deltabf=\thetabf^*-\mubf$ as random variable with mean zero and $\xibf=\mubf - \Bar{\thetabf}$ as difference between the estimated and true mean, furthermore we define the covariance matrices $\Sigmabf,\Sigmabf^*=\frac{1}{t}\sum_{i=1}^t(\thetabf^*(i)-\Bar{\thetabf}^*)(\thetabf^*(i)-\Bar{\thetabf}^*)^\top,\hat{\Sigmabf}=\frac{1}{t}\sum_{i=1}^t(\thetabf(i)-\Bar{\thetabf})(\thetabf(i)-\Bar{\thetabf})^\top$ as the true covariance matrix, the covariance estimated by $\thetabf^*(i)$ and the covariance estimated by $\hat{\thetabf}(i)$. We also define vertical concatenations $\U=[\u_j^\top]_{j\in\{1,...,p\}}^\top$, $\U^*=[\u_j^{*\top}]_{j\in\{1,...,p\}}^\top$ and $\hat{\U}=[\hat{\u}_j^\top]_{j\in\{1,...,p\}}^\top$, with $\{\u_1,...,\u_p\}$, $\{\u^*_1,...,\u^*_p\}$, $\{\hat{\u}_1,...,\hat{\u}_p\}$ being the eigenvectors corresponding to the $p$ largest eigenvalues of $\Sigmabf$, $\Sigmabf^*$ and $\hat{\Sigmabf}$ respectively. Similarly we define $\P=\U\U^\top,\P^*=\U^*\U^{*\top},\hat{\P}=\hat{\U}\hat{\U}^\top$ as the true projection, the projection estimated by the true task parameters $\thetabf^*(i)$ and the projection estimated by $\thetabf(i)$ respectively. For the following parts we need to define the matrix norms: We denote the matrix norm $\norm{\cdot}$ as the spectral norm and $\norm{\cdot}_F$ as the Frobenius norm.
We also require some auxiliary Lemmas:

\begin{lemma}[\cite{Smale2007LearningTE}]
    Let $\thetabf^*(1),...,\thetabf(t)^*\in\mathbb{R}^d$ be vector valued random variables sampled from a distribution $\rho$ with true mean $\mubf$ and $\norm{\thetabf^*(i)}\leq V$, $\forall i \in \{1,...,t\}$. Then the following holds with probability $1-\delta$:

    \begin{equation*}
        \norm{\Bar{\thetabf}^*-\mubf}\leq \frac{2\log(\frac{2}{\delta})V}{t}+\sqrt{\frac{2\log(\frac{2}{\delta})\mathrm{Var_{\max}}}{t}},
    \end{equation*}

    with $\mathrm{Var_{\max}}=\mathbb{E}\left[\norm{\Deltabf}^2\right]=\tr(\Sigmabf)$ as the total variance of distribution $\rho$.
    \label{lemma:mean_concentration}
\end{lemma}




\begin{lemma}[Corollary 5.50 of \cite{vershynin_2012}]
    Consider a subgaussian distribution in $\mathbb{R}^d$ with true covariance $\Sigmabf$ and the covariance $\Sigmabf^*$ estimated from $t$ samples as it was defined above. Let $\delta\in(0,1)$, then we have with probability $1-\delta$:

    \begin{equation*}
        \norm{\Sigmabf-\Sigmabf^*}\leq \sqrt{C\frac{\log(2/\delta)}{t}},
    \end{equation*}

    with $C$ as an absolute constant.

    \label{lemma:covariance_concentration_inequality}
\end{lemma}

\begin{lemma}[Theorem 2 in \cite{yu_davis_kahan_2015}]
    Let $\Sigmabf,\hat{\Sigmabf}\in\mathbb{R}^{d\times d}$ be two symmetric matrices with eigenvalues $\sigma_1\geq...\geq\sigma_d$ and $\hat{\sigma}_1\geq...\geq\hat{\sigma}_d$ respectively. Fix $1\leq p \leq d$ and let $\U=[\u_1,...,\u_p]$ and $\hat{\U}=[\hat{\u}_1,...,\hat{\u}_p]$ with eigenvectors $\u_i$ and $\hat{\u}_i$ of matrices $\Sigmabf$ and $\hat{\Sigmabf}$ respectively. Assume that the eigengap satisfies $\Delta_{\sigma}=\sigma_p-\sigma_{p+1}>0$, then there exists an orthogonal matrix $\O$ such that the the following holds:

    \begin{equation*}
        \norm{\U-\hat{\U}\O}_F \leq \frac{\sqrt{8}\min\left(\sqrt{p}\norm{\Sigmabf-\hat{\Sigmabf}},\norm{\Sigmabf-\hat{\Sigmabf}}_F\right)}{\Delta_{\sigma}},
    \end{equation*}
    
    \label{lemma:davis_kahan_theorem}
\end{lemma}

\begin{proof}[Proof of Lemma \ref{projection_error}]
    For the proof we will use the triangular inequality to express the bound in terms of the true variance along the orthogonal subspace, the projected mean estimation error and the projection estimation error. For the mean estimation error we apply an additional triangular inequality in order to estimate it with respect to the true mean estimation error $\norm{\P^{\perp}(\mubf-\Bar{\thetabf}^*)}$ and the error $\norm{\Bar{\thetabf}^*-\Bar{\thetabf}}$, with the former being a simple concentration bound and the latter being estimated from the oracle inequality for $\thetabf$. We intend to express the projection error with respect to the estimation error on the covariance matrix. Bounding the term $\norm{\P-\hat{\P}}$ requires the Davis-Kahan Theorem. Thus we begin the proof:

    \begin{align}
        \norm{\hat{\P}^{\perp}(\thetabf^*-\Bar{\thetabf})}&\leq \norm{\hat{\P}^{\perp}-{{\P}^{\perp}}} \norm{\thetabf^*-\Bar{\thetabf}}+\norm{\P^{\perp}(\thetabf^*-\Bar{\thetabf})}\\
        &\leq \norm{\hat{\P}^{\perp}-\P^{\perp}}\norm{\thetabf^*-\Bar{\thetabf}}+\norm{\P^{\perp}\Deltabf}+\norm{\P^{\perp}\xibf}\\
        &\leq  2V\norm{\hat{\P}-{{\P}}}+\norm{\P^{\perp}\Deltabf}+\norm{\P^{\perp}\xibf},
        \label{eq:reconstruction_error_estimate}
    \end{align}

    where we used $\P_i^{\perp}-\P_j^{\perp}=\P_i-\P_j$ for all projection matrices $\P_i,\P_j$ in the last inequality.
    We deliver an upper bound on all of these terms separately:
    The second term being straight forward with $\P^{\perp}$ as the true orthogonal projection:

    \begin{equation}
        \mathbb{E}_{\thetabf^*\sim \rho}\left[\norm{\P^{\perp}\Deltabf}^2\right]=\mathrm{Var}_{\rho},
    \end{equation}
    
    with $\mathrm{Var}_{\rho}$ denoting the low variance of distribution $\rho$ along the orthogonal subspace. This holds simply due to our problem setting.

    The last term yields the mean estimation error of tasks along the orthogonal subspace which was similarly discussed in \cite{cella2020meta}:

    \begin{equation}
        \norm{\P^{\perp}(\mubf-\Bar{\thetabf})}\leq \norm{\P^{\perp}(\mubf-\Bar{\thetabf}^*)}+\norm{(\Bar{\thetabf}^*-\Bar{\thetabf})}
        \label{eq:mean_estimation_error}
    \end{equation}

    The first term can simply be bounded by a concentration inequality which was also discussed in Lemma 3 of \cite{cella2020meta} by using Lemma \ref{lemma:mean_concentration} we state with probability of $1-\delta$ that the following holds:

    \begin{equation*}
        \norm{\P^{\perp}(\mubf-\Bar{\thetabf}^*)}\leq \frac{2\log(\frac{2}{\delta})}{t}+\sqrt{\frac{2\log(\frac{2}{\delta})\mathrm{Var}_{\rho}}{t}}.
    \end{equation*}

    By choosing $\delta=1/t$ and taking the expectation value with respect to the task distribution, we have:

\begin{equation*}
        \mathbb{E}_{\thetabf^*\sim\rho}\left[\norm{\P^{\perp}\left(\mubf-\Bar{\thetabf}^*\right)}\right]= \mathcal{O}\left( \frac{2\log(2t)}{t}+\sqrt{\frac{2\log(2t)\mathrm{Var}_{\rho}}{t}}\right).
    \end{equation*}

    We will denote $\epsilon_{\mubf}\coloneqq \frac{2\log(2t)}{t}+\sqrt{\frac{2\log(2t)\mathrm{Var}_{\rho}}{t}}$ for the rest of the proof
    As for the second term in \cref{eq:mean_estimation_error}, we assume that all previously learnt tasks were running for at least $n$ rounds and use the subscript $i\in\{1,...,t\}$ to refer to a given task:

    \begin{align}
        \norm{\Bar{\thetabf}^*-\Bar{\thetabf}}&\leq\max_{i}\norm{\thetabf(i)^*-\thetabf(i)}\\
        &\leq \max_{i}\frac{\norm{\thetabf^*(i)-\thetabf(i)}_{\A_{n}(i)}}{\sqrt{\lambda_{\min}(\A_{n}(i))}}\\
        &\leq \frac{1}{\sqrt{\log(n)}}\left( \sqrt{d\log(1+\frac{n}{d\lambda})+\log(\frac{1}{\delta^2})}+\sqrt{\lambda}V\right),
        \label{eq:estimated_mean_estimated_bound}
    \end{align}

    where we used a linear regression result $\lambda_{\min}(\A_n)\geq\log(n)$ from \cite{Lai_1982}. For the most general case, we will keep $\lambda_{\min}=\min_i\lambda_{\min}(\A_n(i))$
    Choosing $\delta=1/n$, $\lambda=\frac{1}{n V^2}$ and taking the expectation with respect to the arm selection process yields:

    \begin{equation}
        \mathbb{E}\left[\norm{\Bar{\thetabf}^*-\Bar{\thetabf}}\right]\leq\mathcal{O}\left(\frac{1}{\lambda_{\min}}\sqrt{d\log(1+\frac{n^2V^2}{d})+2}+\sqrt{\frac{1}{n}}\right).
        \label{eq:expected_estimated_mean_estimated_bound}
    \end{equation}

    We denote $\beta_{d}\coloneqq\frac{1}{\lambda_{\min}}\sqrt{d\log(1+\frac{n^2V^2}{d})+2}+\sqrt{\frac{1}{n}}$
    We note that this upper bound is independent from the task distribution.


    
     What is left is to upper bound the term $\norm{\P-\hat{\P}}$:

    \begin{align*}
        \norm{\P-\hat{\P}} &= \norm{\U\U^{\top}-\hat{\U}\hat{\U}^\top}\\
        &= \norm{\U\U^{\top}-\hat{\U}\O\O^\top\hat{\U}^\top} \\
        &= \norm{\U\U^{\top} + \hat{\U}\O\U^{\top} - \hat{\U}\O\U^{\top}-\hat{\U}\O\O^\top\hat{\U}^\top}\\
        &= \norm{\hat{\U}\O(\U^{\top}-\O^\top\hat{\U}^\top)+(\U-\hat{\U}\O)\U^{\top}}\\
        &\leq \norm{\hat{\U}\O(\U^{\top}-\O^\top\hat{\U}^\top)}+\norm{(\U-\hat{\U}\O)\U^{\top}}\\
        &\leq 2 \norm{\U-\hat{\U}\O}_F,
    \end{align*}
    
    where we used Cauchy-Schwarz in the last inequality and the fact that $\O$ is a orthogonal matrix and $\U^{\top}\U=\hat{\U}^\top\hat{\U}=\I$. Now we are able to apply Lemma \ref{lemma:davis_kahan_theorem}:

\begin{equation}
    \norm{\P-\hat{\P}} \leq \frac{\sqrt{32}\min\left(\sqrt{p}\norm{\Sigmabf-\hat{\Sigmabf}},\norm{\Sigmabf-\hat{\Sigmabf}}_F\right)}{\Delta_{\sigma}}
    \label{eq:projection_estimation_error}
\end{equation}

    Using the triangular inequality we bound the term $\norm{\Sigmabf-\hat{\Sigmabf}}$:

    \begin{equation}
        \norm{\Sigmabf-\hat{\Sigmabf}}\leq \norm{\Sigmabf-\Sigmabf^*}+ \norm{\Sigmabf^*-\hat{\Sigmabf}}.
        \label{eq:triangular_covariance}
    \end{equation}

    The first term in \cref{eq:triangular_covariance} is a simple concentration inequality for covariance matrices. Using Lemma \ref{lemma:covariance_concentration_inequality} we have with probability $1-\delta$:

    \begin{equation*}
        \norm{\Sigmabf^*-\Sigmabf}\leq \sqrt{\frac{C\log(\frac{2}{\delta})}{t}},
    \end{equation*}

    with an absolute constant $C$.
    Setting $\delta=1/t$ and taking the expectation value yields:

    \begin{equation*}
        \mathbb{E}_{\thetabf^*\sim\rho}\left[\norm{\Sigmabf^*-\Sigmabf}\right]= \mathcal{O}\left(\sqrt{\frac{C\log(2t)}{t}}\right).
    \end{equation*}

    We denote $\epsilon_{\Sigmabf}\coloneqq\sqrt{\frac{C\log(2t)}{t}}$.
    Finally we need to bound $\norm{\Sigmabf^*-\hat{\Sigmabf}}$. We denote $\Thetabf^*=\left[\left(\thetabf^*(i)-\Bar{\thetabf}^*\right)^\top\right]_{\{i\in{1,...,t}\}}$ and $\hat{\Thetabf}=\left[\left(\thetabf(i)-\Bar{\thetabf}\right)^\top\right]_{i\in\{1,...,t\}}$, with vertically concatenated vectors, such that we have:

    \begin{align*}
        \norm{\Sigmabf^*-\hat{\Sigmabf}} &\leq \norm{\Sigmabf^*-\hat{\Sigmabf}}_F\\
        &=\frac{1}{t}\norm{\Thetabf^{*\top}\Thetabf^*-\hat{\Thetabf}^\top\hat{\Thetabf}}_F \\
        &= \frac{1}{t}\norm{\Thetabf^{*\top}\Thetabf^*-\Thetabf^{*\top}\hat{\Thetabf}+\Thetabf^{*\top}\hat{\Thetabf}-\hat{\Thetabf}^\top\hat{\Thetabf}}_F \\
        &= \frac{1}{t}\norm{\Thetabf^{*\top}\left(\Thetabf^*-\hat{\Thetabf}\right)+\left(\Thetabf^{*\top}-\hat{\Thetabf}^\top\right)\hat{\Thetabf}}_F\\
        &\leq \frac{1}{t}\left(\norm{\Thetabf^*}_F+\norm{\hat{\Thetabf}}_F\right)\norm{\Thetabf^*-\hat{\Thetabf}}_F
    \end{align*}

    We can further bound this while also taking the expectation, using:

    \begin{equation*}
        \mathbb{E}\left[\norm{\Thetabf^*-\hat{\Thetabf}}_{F}\right] \leq \mathbb{E}\left[\sqrt{t\beta_{d}^2}+\sqrt{\sum_{i=1}^t\norm{\thetabf^*(i)-\thetabf(i)}^2}\right]\leq2\sqrt{t}\beta_{d}.
    \end{equation*}

    where we used the result of \cref{eq:expected_estimated_mean_estimated_bound}.
    The same estimation can be done for the term $\norm{\Thetabf^*}_F+\norm{\hat{\Thetabf}}_F$:

    \begin{equation*}
         \norm{\Thetabf^*}_F+\norm{\hat{\Thetabf}}_F\leq \sqrt{t}\left(\max_i\norm{\thetabf^*(i)-\Bar{\thetabf}^*}+\max_i\norm{\thetabf(i)-\Bar{\thetabf}}\right] \leq 4\sqrt{t}V
    \end{equation*}

    Thus we conclude:

    \begin{equation*}
        \mathbb{E}\left[\norm{\Sigmabf^*-\hat{\Sigmabf}}\right]\leq 8V\beta_{d}
    \end{equation*}

    Inserting the results into \cref{eq:projection_estimation_error} gives:

    \begin{equation}
        \mathbb{E}\left[\norm{\P-\hat{\P}}\right]\leq \sqrt{32p}\frac{8V\beta_{d}+\epsilon_{\Sigmabf}}{\Delta_{\sigma}}
    \end{equation}
    
    After estimation of every term of our original expression we can summarize it by taking the expectation and applying Jensen's inequality:

    \begin{equation}
        \mathbb{E}_{\theta^*\sim\rho}\left[\mathbb{E}\left[\norm{\hat{\P}^{\perp}(\thetabf^*-\Bar{\thetabf})}\right]\right]= \mathcal{O}\left(\sqrt{\mathrm{Var}_{\rho}+\beta_{d}^2\left(1+64\sqrt{2p}\frac{V^2}{\Delta_{\sigma}}\right)^2+\epsilon_{\mubf}^2+\frac{128p\epsilon_{\Sigmabf}^2V^2}{\Delta_{\sigma}^2}}\right)
    \end{equation}  
\end{proof}

\begin{lemma}{\cite[Lemma 11]{abbasi2011improved}}
    Let $\x_{a_k}$ be a sequence in $\Rset^d$ with $\norm{\x_{a_k}}\leq 1$ and $\B$ defined as usual. Then we have:

    \begin{equation*}
        \sum_{k=1}^n \norm{\x_{a_{k-1}}}^2_{\B^{-1}_{k-1}} \leq 2S_{k-1}^{\lambda_1,\lambda_2}
    \end{equation*}
    \label{lemma:sum_x_norms}
\end{lemma}

\begin{proof}[Proof of \Cref{theorem:proj_ucb_bound}]
    First we denote $\xi \coloneqq \norm{\x}_{\B^{-1}}\norm{\hat{\thetabf}-\thetabf^*}_{\B}$ as exploration term.
    Then we continue by estimating the peudo regret $R(n)$:
        \begin{align*}
            R(n) &= \sum_{k=1}^n \left(\x_{a_{k-1}^*}-\x_{a_{k-1}}\right)^\top\thetabf^* \\
            & \leq \sum_{k=1}^n\x_{a_{k-1}}^\top\hat{\thetabf}_{k-1} + \xi_{k-1} - \x_{a_{k-1}}^\top\thetabf^* \\
            &\leq \sum_{k=1}^n\x_{a_{k-1}}^\top\hat{\thetabf}_{k-1} + \xi_{k-1} - \x_{a_{k-1}}^\top\hat{\thetabf}_{k-1} + \xi_{k-1} \\
            &= \sum_{k=1}^n 2\xi_{k-1} \\
            &\leq \sum_{k=1}^n \biggl(\sqrt{p\log(\frac{1}{\delta}+\frac{k-1}{p\lambda_2\delta})+q\log(\frac{1}{\delta}+\frac{k-1}{q\lambda_1\delta})}+\sqrt{\lambda_2}V+\frac{\lambda_1}{\sqrt{\lambda_2}}W\biggr) \norm{\x_{a_{k-1}}}_{\B^{-1}_{k-1}}\\
            &\leq \biggl(\sqrt{p\log(\frac{1}{\delta}+\frac{n}{p\lambda_2\delta})+q\log(\frac{1}{\delta}+\frac{n}{q\lambda_1\delta})}+\sqrt{\lambda_2}V+\frac{\lambda_1}{\sqrt{\lambda_2}}W\biggr)\\
            &\sqrt{n\sum_{k=1}^n\norm{\x_{a_{k-1}}}^2_{\B^{-1}_{k-1}}} \\
                &\leq \biggl(\sqrt{p\log(1+\frac{n}{p\lambda_2})+q\log(1+\frac{n}{q\lambda_1})+\log(\frac{1}{\delta^2})} +\sqrt{\lambda_2}V+\frac{\lambda_1}{\sqrt{\lambda_2}}W\biggr)\\
         &\sqrt{2n\biggl(p\log(1+\frac{n}{p\lambda_2})+q\log(1+\frac{n}{q\lambda_1})\biggr)}
        \end{align*}
        The first and second inequality make use of the OFUL principle and the definition of the UCB function.  
        We used Lemma \ref{lemma:confidence_set_bound} in the third inequality and Lemma \ref{lemma:sum_x_norms} in the last inequality. This regret holds with probability $1-\delta$.

        \begin{align*}
            \mathbb{E}_{\thetabf^*\sim\rho}\left[\mathbb{E}\left[R(n)\right]\right] &\leq \left(\sqrt{p\log(1+\frac{nV^2}{p})+q\log(1+\frac{n\sqrt{Y}}{q})+\log(\frac{1}{\delta^2})} +1+V\right)\\
         &\sqrt{2n\left(p\log(1+\frac{nV^2}{p})+q\log(1+\frac{n\sqrt{Y}}{q})\right)}
        \end{align*}
        
        We obtain the final results by setting $\delta=1/n$, take the expectation value, followed by an additional expectation value with respect to the task distribution: $\mathbb{E}_{\thetabf^*\sim\rho}\left[\mathbb{E}\left[R(n)\right]\right]$, setting $\lambda_1=\frac{1}{\sqrt{Y}}$, $\lambda_2=\frac{1}{V^2}$ and application of Jensen's inequality.
    \end{proof}

\section{Proof of \texorpdfstring{\Cref{theorem:proj_ts_regret}}{}}
The following proofs are adapted from \cite{agrawal:linearTS} and are required to finish the proof on the regret bound.
Before proceeding, we define the concept of a saturated arm, which is basically a measurement of the required exploration for any arm. 
\begin{definition}
    We call an arm $a$ saturated if $g_n\norm{\x_a}_{\B^{-1}}<l_n\norm{\x_{a^*}}_{\B^{-1}}$ and unsaturated otherwise, with $g_n=\sqrt{2d+6\log(n)}v+l_n$. The set of saturated arms at round $k$ is denoted as $\mathcal{C}_k$.
\end{definition}

We will also utilize the following Lemma from \cite{hanson_wright_by_hsu}, which is a special case of the inequality in \cite{hanso_wright_1971}:

\begin{lemma}[Proposition 1.1 in \cite{hanson_wright_by_hsu}]
    Let $\x\in\mathbb{R}^d$ be a $d$-dimensional standard normal variable and $\C\in\mathbb{R}^{d\times d}$ a matrix. Then we have for all $t>0$:

\begin{equation*}
    \mathrm{Pr}\left(\norm{\C\x}^2>\tr(\C^{\top}\C)+2\sqrt{\tr(\left(\C^{\top}\C\right)^2)t}+2\norm{\C^{\top}\C}t\right)\leq e^{-t}
\end{equation*}
    \label{lemma:hanson_wright}
\end{lemma}

\begin{proof}[Proof of Lemma \ref{lemma:Event_bound}]
    The probability of event $E_r$ is determined using Lemma \ref{lemma:confidence_set_bound}: we have with probability $1-\delta$:
    
    \begin{align*}
        |\x^\top_{a}\left(\thetabf^*-\hat{\thetabf}_k\right)|&\leq \norm{\x_{a}}_{\B^{-1}_{k}}\norm{\thetabf^*-\hat{\thetabf}_k}_{\B_{k}}\\
        &\leq \norm{\x_{a}}_{\B^{-1}_{k}}\left(\sqrt{S_{k}^{\lambda_1,\lambda_2}+\log(\frac{1}{\delta^2})}+\frac{\lambda_2}{\sqrt{\lambda_1}}W+\sqrt{\lambda_1}V\right),
    \end{align*}
    
    by substituting $\delta \xrightarrow[]{}\frac{\delta}{n^2}$ and further upper bounding $S_{k}^{\lambda_1,\lambda_2}$ we get:
    
    
    
    \begin{align*}
        \sqrt{S_{k}^{\lambda_1,\lambda_2}+\log(\frac{n^2}{\delta^2})}&= \sqrt{p\log(1+\frac{k}{p\lambda_2})+q\log(1+\frac{k}{q\lambda_1})+\log(\frac{n^2}{\delta^2})}\\
        &\leq \sqrt{p\log(n\left(\frac{n}{\delta}\right)^{2/d})+q\log(n\left(\frac{n}{\delta}\right)^{2/d})}\\
        &\leq \sqrt{\log(\frac{1}{\delta}) (d+2)\log(n)},
    \end{align*}

    and therefore we have:

    \begin{align*}
        |\x^\top_{a}\left(\thetabf^*-\hat{\thetabf}_k\right)|&\leq\left(\sqrt{\log(\frac{1}{\delta}) (d+2)\log(n)} + \frac{\lambda_2}{\sqrt{\lambda_1}}W+\sqrt{\lambda_1}V\right)\norm{\x_{a}}_{\B^{-1}_{k}}\\
        &\leq \sqrt{2\log(\frac{1}{\delta}) (d+2)\log(n)+2K^2}\norm{\x_{a}}_{\B^{-1}_{k}},
    \end{align*}
    
    with $K=\frac{\lambda_2}{\sqrt{\lambda_1}}W+\sqrt{\lambda_1}V$. Since we substituted $\delta\xrightarrow[]{}\frac{\delta}{n^2}$, this event has a probability of at least $1-\frac{\delta}{n^2}$. \\
    For proof of the bound on the probability of event $E_{\theta}$ we have for all $a\in\mathcal{A}_k$:

    \begin{align*}
        \left|\x_{a}^\top\left(\hat{\thetabf}_k-\Tilde{\thetabf}_k\right)\right| &= \left|\x_{a}^{\top}\B_k^{-\frac{1}{2}}\B_k^{\frac{1}{2}}\left(\Tilde{\thetabf}_k-\Hat{\thetabf}_k\right)\right| \\
        &\leq v \sqrt{\x_{a}^{\top}\B_k^{-1}\x_{a}}\norm{\frac{1}{v}\B_k^{\frac{1}{2}}\left(\Tilde{\thetabf}_k-\Hat{\thetabf}_k\right)}
    \end{align*}.

    By definition, the term $\frac{1}{v}\B_k^{\frac{1}{2}}\left(\Tilde{\thetabf}_k-\Hat{\thetabf}_k\right)$ is $d$-dimensional standard normal variable, such that we can apply \Cref{lemma:hanson_wright}, where we set $\C=\I$ and $t=2\log(n)$:

    \begin{align*}
        \mathrm{Pr}\left(\norm{\frac{1}{v}\B_k^{\frac{1}{2}}\left(\Tilde{\thetabf}_k-\Hat{\thetabf}_k\right)}>\sqrt{d+\sqrt{8d\log(n)}+4\log(n)}\right) \leq \frac{1}{n^2}.
    \end{align*}

    Thus the following inequality holds with probability of at least $1-\frac{1}{n^2}$ for all $a\in\mathcal{A}_k$:

    \begin{align*}
        \left|\x_{a}^\top\left(\hat{\thetabf}_k-\Tilde{\thetabf}_k\right)\right| &\leq v\norm{\x_a}_{\B^{-1}_k}\sqrt{d+\sqrt{8d\log(n)}+4\log(n)}\\
        &\leq v\norm{\x_a}_{\B^{-1}_k}\sqrt{2d+6\log(n)},
    \end{align*}

    where we used the inequality of arithmetic and geometric means in the last step.
\end{proof}

\begin{lemma}
    For any filtration $\mathcal{F}_{k-1}$ such that $E_r$ is true, we have:
    \begin{equation*}
        \mathrm{Pr}(\x_{a^*_k}^\top\Tilde{\thetabf}_k>\x_{a^*_k}^\top\thetabf^*+l_n\norm{\x_{a^*_k}}_{\B^{-1}_k})\geq c_n
    \end{equation*}

    and:
    
    \begin{equation*}
        \mathrm{Pr}(a_k\in\mathcal{C}_k|\mathcal{F}_{k-1})\leq \frac{1}{c_n}\mathrm{Pr}(a_k\notin\mathcal{C}_k|\mathcal{F}_{k-1}) +\frac{1}{c_n n^2},
    \end{equation*}

    with $c_n=\frac{1}{4e\sqrt{\pi n^{\alpha}}}$.
\end{lemma}

\begin{proof}
        Assuming the event $E_r$ holds and $\x_{a^*}^\top\Tilde{\thetabf}$ is a Gaussian random variable with mean $\x_{a^*}^\top\Hat{\thetabf}$ and variance $v\norm{x_{a^*}}_{\B^{-1}}$, we can apply the anti-concentration inequality such that:
        
        \begin{align*}
            \mathrm{Pr}(\x_{a^*_k}^\top\Tilde{\thetabf}_k\geq\x_{a^*_k}^\top\Hat{\thetabf}_k+l_n\norm{\x_{a^*_k}}_{\B^{-1}_k}|\mathcal{F}_{k-1}) &=\mathrm{Pr}\left(\frac{\x_{a^*_k}^\top\left(\Tilde{\thetabf}_k-\Hat{\thetabf}_k\right)}{v\norm{\x_{a^*_k}}_{\B^{-1}_k}}\geq\frac{\x_{a^*_k}^\top\left(\thetabf^*-\hat{\thetabf}_k\right)+l_n\norm{\x_{a^*_k}}_{\B^{-1}_k}}{v\norm{\x_{a^*_k}}_{\B^{-1}_k}}\bigg|\mathcal{F}_{k-1}\right)\\
            &\geq \frac{1}{4\sqrt{\pi}}\exp(-Z^2),
        \end{align*}
    with
        \begin{align*}
            Z &= \frac{\x_{a_k^*}^\top\left(\thetabf^*-\hat{\thetabf}_k\right)+l_n\norm{x_{a^*_k}}_{\B^{-1}_k}}{v\norm{x_{a^*_k}}_{\B^{-1}_k}} \\
            &\leq \frac{2l_n\norm{x_{a^*_k}}_{\B^{-1}_k}}{v\norm{x_{a^*_k}}_{\B^{-1}_k}}\\
            &\leq \frac{2\sqrt{2\log(\frac{1}{\delta}) (d+2)\log(n)+2K^2}}{4\sqrt{\log(\frac{1}{\delta})\frac{d+2}{\alpha}}} \\
            &\leq \sqrt{\frac{\alpha}{2}\log(n)+\frac{\alpha K^2}{8\log(\frac{1}{\delta})(d+2)}} \\
            &\leq \sqrt{\frac{\alpha}{2}\log(n)+1}.
        \end{align*}
        Thus we have 
        
        \begin{equation*}
            \mathrm{Pr}(\x_{a^*_k}^\top\Tilde{\thetabf}_k\geq\x_{a^*_k}^\top\Hat{\thetabf}_k+l_n\norm{\x_{a^*_k}}_{\B^{-1}_k}|\mathcal{F}_{k-1}) \geq  \frac{1}{4e\sqrt{\pi n^{\alpha}}}
        \end{equation*}
        
        The proof of the second inequality is provided in Lemma 3 of \cite{agrawal:linearTS}.
        
    \end{proof}

\begin{lemma}
    Let $\mathrm{regret}_k=(\x_{a^*}^\top-\x_{a_k}^\top)\thetabf^*$ be defined as the instantaneous regret at round $k$ and $\mathrm{regret}'_k=\mathrm{regret}_kI(E_r)$. \\
    Define 
    \begin{equation*}
        X_k = \mathrm{regret'_k} - \frac{g_n}{c_n}I\left(a_k\notin\mathcal{C}\right)\norm{x_{a^*}}_{\B^{-1}_k} - \frac{2g_n}{c_n n^2}- \frac{2g_n^2}{l_n}\norm{x_{a_k}}_{\B^{-1}_k}      
    \end{equation*}
    and
    \begin{equation*}
        Y_k = \sum_{t=1}^k X_k,
    \end{equation*}
    then $(Y_k;k=1,...,n)$ is a super-martingale process with respect to filtration $\mathcal{F}_{k-1}$.
    
\end{lemma}

\begin{proof}
        The proof is provided in Lemma 4 of \cite{agrawal:linearTS}.
\end{proof}

 \begin{proof}[Proof of \Cref{theorem:proj_ts_regret}]
        Each value in $X_k$ is bounded by $\frac{2g_n^{2}}{c_nl_n}$ which implies a bounded difference on the super-martingale $Y_k$ with $|Y_k-Y_{k-1}|\leq\frac{8g_n^{2}}{c_nl_n}$, allowing us to apply Azuma-Hoeffding's inequality during the proof. Thus we have with probability $1-\frac{\delta}{2}$:
        \begin{align*}
            \sum_{k=1}^n\mathrm{regret}'_k&\leq \sum_{k=1}^n \left(\frac{g_n}{c_n}I(a_k\notin \mathcal{C}_k\norm{x_{a^*}}_{\B_k^{-1}})\right)+\frac{2g_n}{c_n n^2}\\
            &+\frac{2g_n^{2}}{c_n l_n}\sum_{k=1}^n\norm{x_{a_k}}_{\B_k^{-1}}+\frac{8g_n^{2}}{c_n l_n}\sqrt{2n\log(\frac{2}{\delta})}\\
            &\leq \sum_{k=1}^n \left(\frac{g_n^{2}}{c_n l_n}I(a_k\notin \mathcal{C}_k\norm{x_{a_k}}_{\B_k^{-1}})\right)+\frac{2g_n}{c_n n^2}\\
            &+\frac{2g_n^{2}}{c_n l_n}\sum_{k=1}^n\norm{x_{a_k}}_{\B_k^{-1}}+\frac{8g_n^{2}}{c_n l_n}\sqrt{2n\log(\frac{2}{\delta})} \\
            &\leq \frac{3g_n^{2}}{c_n l_n}\sum_{k=1}^n\norm{x_{a_k}}_{\B_k^{-1}}+\frac{2g_n}{c_n n^2}+\frac{8g_n^{2}}{c_n l_n}\sqrt{2n\log(\frac{2}{\delta})} \\
            &\leq \frac{3g_n^{2}}{c_n l_n}\sqrt{2nS_{n}^{\lambda_1,\lambda_2}}+\frac{2g_n}{c_n n^2}+\frac{8g_n^{2}}{c_n l_n}\sqrt{2n\log(\frac{2}{\delta})}\\
            &= \frac{g_n^{2}}{c_n l_n}\left(\sqrt{18nS_{n}^{\lambda_1,\lambda_2}}+\sqrt{128n\log(\frac{2}{\delta})}\right)+\frac{2g_n}{c_n n^2} \\
            &= \left(\frac{l_n+2\sqrt{2d+6\log(n)}v+(2d+6\log(n))v^2/l_n}{c_n}\right)\left(\sqrt{18nS_{n}^{\lambda_1,\lambda_2}}+\sqrt{128n\log(\frac{2}{\delta})}\right)+\frac{2g_n}{c_n n^2} \\
            &= \mathcal{O}\left(\left(2\sqrt{\frac{2d^2\log(\frac{1}{\delta})+6d\log(\frac{1}{\delta})\log(n)}{\alpha}}+\frac{2d^{\frac{3}{2}}\sqrt{\log(\frac{1}{\delta})}}{\alpha\sqrt{\log(n)}}+\frac{6\sqrt{d\log(\frac{1}{\delta})\log(n)}}{\alpha}\right)\sqrt{n^{1+\alpha}S_n^{\lambda_1,\lambda_2}}\right)
        \end{align*}
        We used our definition for saturated arms in the second inequality and Lemma \ref{lemma:sum_x_norms} in the fourth inequality.
        Now similar as done in \Cref{theorem:proj_ucb_bound} we set $\delta=\frac{1}{n}$ and take the expectation value. Additionally we set $\alpha=\frac{1}{\log(n)}$:

        \begin{equation*}
            \mathbb{E}\left[\sum_{k=1}^n\mathrm{regret}'_k\right] = \mathcal{O}\left(\left(d^{\frac{3}{2}}\log(n)+\sqrt{d}\log(n)^2\right)\sqrt{nS_{n}^{\lambda_1,\lambda_2}}\right)
        \end{equation*}
        
        Inserting $\lambda_2=1/{V^2}$ and $\lambda_1=\frac{1}{\sqrt{Y}}$, while taking the second expectation value with respect to the task distribution and applying Jensen's inequality, gives the final result:
        
        \begin{align*}
            \mathcal{R}(n)&= \mathcal{O}\biggl(\left(d^{\frac{3}{2}}\log(n)+\sqrt{d}\log(n)^2\right) \sqrt{n\left(p\log(1+\frac{nV^2}{p})+q\log(1+\frac{n \sqrt{Y}}{q})\right)}\biggr)
        \end{align*}
    \end{proof}

\end{document}